\documentclass[%
 reprint,
 amsmath,amssymb,
 aps,
 pre,
]{revtex4-2}

\usepackage{graphicx}
\usepackage{dcolumn}
\usepackage{bm}
\usepackage{physics}              
\usepackage{xcolor}               
\usepackage{multirow}
\usepackage{setspace}
\usepackage{enumerate}

\usepackage[ruled,linesnumbered]{algorithm2e}

\usepackage{amsthm}
\newtheorem{lemma}{Lemma}
\newtheorem{theorem}{Theorem}

\usepackage[colorlinks,linkcolor=red!90!black,citecolor=green!75!black]{hyperref}

\newcommand{\setzero}{\setlength{\itemsep}{0pt}}

\DeclareMathOperator{\KL}{KL}
\DeclareMathOperator{\Var}{Var}
\DeclareMathOperator{\ESS}{ESS}
\DeclareMathOperator{\CESS}{CESS}

\newcommand{\D}{\mathrm{d}} 
\newcommand{\Ge}{\geqslant}
\newcommand{\Le}{\leqslant}
\newcommand{\T}{\top}
\newcommand{\ep}{\varepsilon}

\begin{document}

\title{Jeffreys Flow: Robust Boltzmann Generators for Rare Event Sampling via Parallel Tempering Distillation}

\author{Guang Lin}
\email{guanglin@purdue.edu}
\author{Christian Moya}
\email{cmoyacal@purdue.edu}
\author{Di Qi}
\email{qi117@purdue.edu}
\author{Xuda Ye}
\email{ye481@purdue.edu}
\affiliation{Department of Mathematics, Purdue University, West Lafayette, IN 47907, USA}

\date{\today}

\begin{abstract}
Sampling physical systems with rough energy landscapes is hindered by rare events and metastable trapping. While Boltzmann generators already offer a solution, their reliance on the reverse Kullback--Leibler divergence frequently induces catastrophic mode collapse, missing specific modes in multi-modal distributions. Here, we introduce the Jeffreys Flow, a robust generative framework that mitigates this failure by distilling empirical sampling data from Parallel Tempering trajectories using the symmetric Jeffreys divergence. This formulation effectively balances local target-seeking precision with global modes coverage. We show that minimizing Jeffreys divergence suppresses mode collapse and structurally corrects inherent inaccuracies via distillation of the empirical reference data. We demonstrate the framework's scalability and accuracy on highly non-convex multidimensional benchmarks, including the systematic correction of stochastic gradient biases in Replica Exchange Stochastic Gradient Langevin Dynamics and the massive acceleration of exact importance sampling in Path Integral Monte Carlo for quantum thermal states.
\end{abstract}

\maketitle

\section{Introduction}
\label{section: introduction}

Rare event sampling is a central challenge in statistical mechanics and computational physics~\cite{bucklew2004introduction, rubino2009rare, rubino2009introduction}. Specifically, consider a target distribution $\pi(x) \propto e^{-U(x)}$ on $x\in\mathbb{R}^d$, where the potential function $U(x)$ exhibits multiple local modes separated by high energy barriers. Classical Monte Carlo methods, such as Metropolis--Hastings \cite{chib1995understanding, robert2009metropolis}, Hamiltonian Monte Carlo (HMC) \cite{betancourt2017conceptual, chen2014stochastic}, and Langevin dynamics \cite{meyn2012markov},
routinely suffer from poor performance, since the generated samples tend to become trapped within a few localized basins. This difficulty is formalized by the Eyring--Kramers formula~\cite{eyring1935activated, kramers1940brownian}, which dictates that the transition probability between metastable modes decays exponentially with the height of the intervening energy barrier~\cite{bouchet2016generalisation, lee2022non}, rendering such transitions as \emph{rare events} for standard Monte Carlo methods.

To simulate rare events, numerous enhanced Monte Carlo methods have been developed: Umbrella Sampling~\cite{torrie1977nonphysical, virnau2004calculation, kastner2011umbrella}, Simulated Annealing~\cite{van1987simulated, bertsimas1993simulated, nikolaev2010simulated}, Sequential Monte Carlo (SMC)~\cite{doucet2001introduction, cappe2007overview, chopin2020introduction, wills2023sequential}, Transition Path Sampling~\cite{dellago2002transition}, Metadynamics~\cite{laio2002escaping}, and Parallel Tempering (PT)~\cite{earl2005parallel, miasojedow2013adaptive, syed2022non, deng2023non}. Among these, PT is widely adopted because it is simple and easy to implement. The core mechanism of PT is to simulate multiple replicas of the system in parallel across a temperature ladder. By frequently swapping configurations between adjacent replicas, PT allows low-temperature simulations to inherit the global ergodicity inherent to the high-temperature limits.

The flow-based generative models emerge as another promising approach for sampling complex target distributions. Empowered by the rapid advancement of GPU architectures, these methods employ a sequence of invertible normalizing flows to push forward a simple base distribution $\pi_0(x)$ towards the target distribution $\pi(x)$. This framework, known as the Boltzmann generator~\cite{noe2019boltzmann,wu2020stochastic}, is conceptually straightforward and primarily relies on minimizing the reverse Kullback--Leibler (KL) divergence. Formally, the goal is to find a transport map $F$ that minimizes $D_{\text{KL}}(F_{\#} \pi_0 \| \pi)$, where $F_{\#} \pi_0$ denotes the pushforward distribution. In practice, $F$ is composed of invertible layers parameterized by expressive architectures such as RealNVP~\cite{dinh2016density} or Neural Spline Flows~\cite{durkan2019neural}.

During inference, a trained Boltzmann generator has unique advantages over classical Monte Carlo methods. It circumvents repeated evaluations of the potential gradient, by leveraging the high memory bandwidth of GPUs to instantly generate massive statistically independent samples. Most importantly, Boltzmann generators produce unbiased samples from the target distribution via importance sampling reweighting.

Although Boltzmann generators have been widely applied in computational physics~\cite{wirnsberger2020targeted,abbott2023normalizing,kohler2024transferable,coretti2024boltzmann,schebek2025scalable}, their reliability in rare event sampling depends heavily on the loss function. When the target distribution $\pi(x)$ has multiple metastable modes separated by high energy barriers, the pushforward samples of the trained flow often get trapped in a local basin, missing other significant modes. This phenomenon, known as \emph{mode collapse}~\cite{midgley2022flow, midgley2023learning}, remains the core challenge in designing efficient algorithms for rare event sampling.

\begin{figure}
    \centering
    \includegraphics[width=0.48\textwidth]{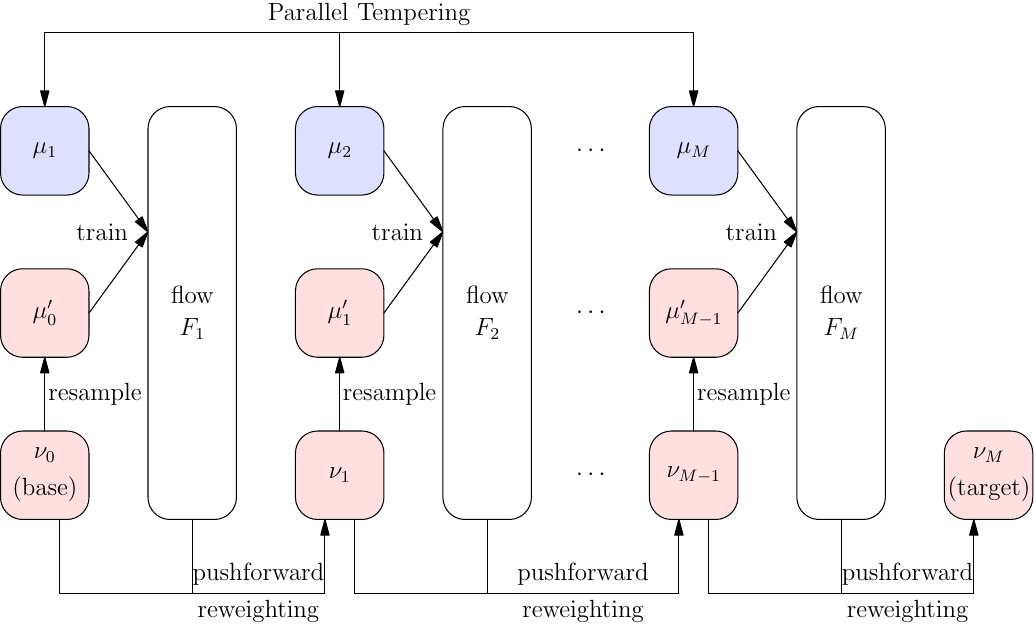}
    \caption{The overall architecture of the Jeffreys Flow. PT acts as the principal guide across the temperature ladder, providing reference samples ($\mu_1, \dots, \mu_M$) to train the sequence of normalizing flows ($F_1, F_2, \dots, F_M$). Simultaneously, intermediate states ($\nu_0, \nu_1, \dots, \nu_M$) are progressively pushed forward and reweighted, intrinsically averting mode collapse without relying solely on the target potential function.}
    \label{figure: architecture}
\end{figure}

We propose a new sampling strategy using the Jeffreys Flow, a Boltzmann generator specifically tailored for rare event sampling. As the name suggests, the Jeffreys Flow adopts the Jeffreys divergence (i.e., the symmetrized KL divergence) as its loss function. Because computing the forward KL divergence requires samples from the target distribution, our method utilizes reference samples generated via PT to train the normalizing flow. Thus, the Jeffreys Flow distills the sampling knowledge from PT. Like standard Boltzmann generators, the Jeffreys Flow maintains the benefits of generating statistically independent samples and allows unbiased reweighting. Most importantly, rather than forcing the flow to learn the complex landscape solely from the target potential, the Jeffreys Flow directly incorporates target samples, thereby impeding the severe mode collapse in rare event sampling. For visual overview of the Jeffreys Flow, see Figure~\ref{figure: architecture}.

While the original Boltzmann generator~\cite{noe2019boltzmann} also utilized the Jeffreys divergence, it did so only in the early stages of training with limited target samples. In contrast, our Jeffreys Flow leverages PT samples to guide the entire training pipeline. Although obtaining reference samples via PT incurs an upfront computational cost, these samples provide theoretical guarantees for training the flow across complex, multi-modal landscapes. Once the training is complete, the expensive PT simulation can be entirely discarded, allowing the trained flow to instantaneously generate statistically independent samples across the entire predefined temperature ladder in just a few computational steps.

Beyond its intuitive design, the Jeffreys Flow is supported by stringent theoretical confirmation. Theorem~\ref{theorem: correct} establishes that the pushforward distributions $(\nu_k)$ generated by minimizing the Jeffreys divergence are strictly closer to the target distribution than the empirical reference samples $(\mu_k)$, underscoring the capacity of the Jeffreys Flow to successively correct inaccuracies inherent in the PT samples. In addition, Theorem~\ref{theorem: concentration} derives a rigorous concentration inequality showing that the probability of mode collapse diminishes to an arbitrarily small level when the Jeffreys divergence is sufficiently minimized. Together, these analytical results confirm that the Jeffreys Flow provides a principled and direct mechanism to impede mode collapse.

In the numerical tests, we evaluate the performance of the Jeffreys Flow on benchmark multi-modal distributions up to 16 dimensions. Then we apply the Jeffreys Flow to two challenging applications: Replica Exchange Stochastic Gradient Langevin Dynamics (reSGLD)~\cite{deng2020accelerating,li2023fast,lin2023b} for finite-sum potentials, and Path Integral Monte Carlo (PIMC)~\cite{herman1982path,marx1996ab,ceriotti2010efficient,schoof2011configuration} for quantum thermal sampling. In the reSGLD, we utilize importance weights to correct the stochastic gradient bias and enhance target accuracy. In the PIMC, we generate PT samples from the cheap classical Boltzmann distribution instead of the expensive ring-polymer quantum equilibrium. A low-dimensional flow is then trained on only the low-frequency normal modes, while the full quantum potential is used solely for importance sampling reweighting to guarantee unbiasedness.

The paper is organized as follows. Section~\ref{section: related works} below reviews related flow-based sampling methods. Section~\ref{section: single step} formalizes the Jeffreys Flow in a single step and proves that the Jeffreys divergence mitigates mode collapse. Section~\ref{section: sequential distillation} explains the sequential distillation framework using Parallel Tempering samples. Sections~\ref{section: reSGLD} and \ref{section: PIMC} apply the framework to reSGLD and PIMC, respectively. Section~\ref{section: numerical experiments} presents validating numerical experiments, and Section~\ref{section: conclusion} concludes with future applications.

\section{Related Work}
\label{section: related works}

To mitigate mode collapse in Boltzmann generators, numerous methods modify the loss function away from the standard reverse KL. For instance, \cite{falkner2023conditioning} introduces a conditional reverse KL constrained by predefined reaction coordinates, while \cite{dibak2022temperature} trains a single reverse KL-based flow parameterized by continuous temperature. Adaptive annealing strategies are also common: \cite{wang2025mitigating} uses Effective Sample Size (ESS) to update the temperature ladder, and \cite{matthews2022continual} integrates the reverse KL into Sequential Monte Carlo (SMC) frameworks. Furthermore, \cite{qiu2024efficient} explicitly replaces the reverse KL with the $L^2$-Wasserstein distance between probability densities. While these approaches force the flow to learn the multi-modal target distribution (often relying heavily on ESS to control the tempering ladder), an alternative strategy replaces the reverse KL with the forward KL~\cite{gabrie2022adaptive, abernethy2024flow}. Although forward KL-based methods exhibit strong mode-covering behavior, they inherently lack the precise physical accuracy provided by energy-based objectives.

Ultimately, Parallel Tempering (PT)~\cite{earl2005parallel, syed2022non} remains the standard benchmark for rare event sampling. While it may require many replicas across different temperatures and lacks a bias-correction mechanism such as importance weighting, its effectiveness relies primarily on maintaining a replica exchange acceptance rate of approximately 20\% to 40\%. Consequently, if training a normalizing flow costs significantly more than running PT, such as when the network scales poorly with dimensionality or requires many flow steps, the generative approach loses its advantage.

Therefore, Jeffreys Flow does not replace PT. Instead, it leverages PT samples from both the base and target distributions to directly guide the training of the normalizing flow.
In machine learning, this distills the PT data to improve sampling efficiency and accuracy. The concept of distillation is common in generative modeling---particularly in image processing and large language models~\cite{luhman2021knowledge, zhou2024simple, xie2024distillation, fu2025moflow}---and often relies on Score-Based Diffusion~\cite{song2021maximum, batzolis2021conditional} or Flow Matching~\cite{lipman2022flow, chen2023flow}. By contrast, the Jeffreys Flow simply uses the Jeffreys divergence, the sum of the forward and reverse KL divergences, as its loss function.

Finally, we point out that the Jeffreys divergence has been well-established in statistical inference~\cite{moreno2003kullback, bayarri2008generalization, yao2011symmetric, sharma2021clustering} and image processing~\cite{nishii2006image, han2020active}. By symmetrizing the KL divergence, it establishes a stable metric that balances the local mode-seeking precision of the reverse KL with the global mode-covering properties of the forward KL, avoiding the asymmetry inherent in using either direction alone. Consequently, the Jeffreys divergence is an ideal loss function for rare event sampling.

\section{Single-Step Jeffreys Flow}
\label{section: single step}

The goal of a flow-based generative model is to train a normalizing flow $F:\mathbb{R}^d \to \mathbb{R}^d$ that maps the base distribution $\pi_0(x) \propto e^{-U_0(x)}$ to the target distribution $\pi_1(x) \propto e^{-U_1(x)}$, i.e., $F_{\#} \pi_0 = \pi_1$. Given access to the potential functions $U_0(x)$ and $U_1(x)$, along with empirical reference distributions $\mu_0$ and $\mu_1$ that approximate $\pi_0$ and $\pi_1$, we use the Jeffreys divergence to construct a robust training objective and prove it outperforms the standard forward and reverse KL divergences.

\subsection{Construction of Jeffreys Divergence}

To facilitate $F_{\#} \pi_0 = \pi_1$, the KL divergence is a natural choice for the loss function. We begin with the reverse KL divergence $D_{\KL}(F_\#\pi_0\|\pi_1)$, which requires the probability density of the pushforward distribution $F_{\#} \pi_0$:
\begin{equation}
	(F_\# \pi_0)(y) = \frac{\pi_0(x)}{|\det \nabla F(x)|},\quad y = F(x).
	\label{eq: Fpi density}
\end{equation}
Substituting \eqref{eq: Fpi density} into the KL divergence and omitting the normalizing constants yields
\begin{align*}
	D_{\KL}(F_\#\pi_0\|\pi_1) &= \mathbb E_{X_0 \sim \pi_0} \Big[ U_1(F(X_0)) - U_0(X_0) \\
	&\quad - \log|\det\nabla F(X_0)| \Big].
\end{align*}
By replacing the base distribution $\pi_0$ with its empirical approximation $\mu_0$, we derive the loss function
\begin{align}
	D_{\KL}(F_\#\pi_0\|\pi_1) &\approx \mathbb E_{X_0 \sim \mu_0} \Big[ U_1(F(X_0)) - U_0(X_0) \notag \\
	&\quad - \log|\det\nabla F(X_0)| \Big],
	\label{eq: rKL}
\end{align}
which requires no samples from the distribution $\pi_1$.

Similarly, we derive the forward KL divergence as
\begin{align}
	D_{\KL}(\pi_1\|F_\#\pi_0) 
	&\approx \mathbb E_{X_1 \sim \mu_1} \Big[ U_0(F^{-1}(X_1)) - U_1(X_1) \notag \\
	&\quad - \log|\det\nabla F^{-1}(X_1)| \Big],
	\label{eq: fKL}
\end{align}
As a consequence, we form the Jeffreys divergence as a weighted sum of the reverse and forward KL divergences:
\begin{equation}
	\mathcal L_J[F] = \lambda_0 D_{\KL}(F_\#\pi_0\|\pi_1) + \lambda_1 D_{\KL}(\pi_1\|F_\#\pi_0),
	\label{eq: Jdiv}
\end{equation}
where $\lambda_0$ and $\lambda_1$ are hyperparameters. 
The KL divergence components in \eqref{eq: Jdiv} can be replaced with the Rényi divergence $D_\alpha(P\|Q)$ to further enhance the robustness of the model against mode collapse. Unlike the standard KL divergence, which uses the logarithmic expectation $\mathbb E_P[\log(P/Q)]$, the Rényi divergence uses power-law moment $\mathbb E_P[(P/Q)^{\alpha-1}]$, recovering the KL divergence exactly at the limit $\alpha \to 1$. This scaling mechanism amplifies gradients for extreme particles situated in the isolated modes or distribution tails. Therefore, the Rényi objective requires abundant training samples to stabilize the high-variance weights, as it fosters a more comprehensive coverage of the complex multi-modal target landscape.

\subsection{Jeffreys Divergence Mitigates Mode Collapse}

The reverse KL divergence $D_{\KL}(F_\#\pi_0\|\pi_1)$ is capable of generating high-fidelity samples, but it can be prone to mode collapse. This critical limitation arises because the reverse KL is inherently mode-seeking; if the generated distribution $F_\#\pi_0$ captures only a subset of the modes of $\pi_1$ while ignoring the rest, the associated divergence penalty remains negligible.

In contrast, incorporating the forward KL effectively mitigates mode collapse. The divergence $D_{\KL}(\pi_1\|F_\#\pi_0)$ strongly penalizes missing mass, growing unbounded if the generated distribution $F_\#\pi_0$ fails to cover any region where  $\pi_1$ has significant mass. However, training a flow using only the forward KL has a key drawback: the fidelity of $F_\#\pi_0$ is limited by the accuracy of the empirical distribution $\mu_1$ relative to the true target $\pi_1$.

For theoretical analysis, we assume that the empirical base distribution is exact, i.e., $\mu_0 = \pi_0$, while $\mu_1$ remains an imperfect approximation of $\pi_1$.
Under this assumption, the Jeffreys divergence $\mathcal L_J[F]$ defined in \eqref{eq: Jdiv} evaluates exactly to
\begin{equation*}
	\mathcal L_J[F] = \lambda_0 D_{\KL}(F_\#\pi_0\|\pi_1) + 
	\lambda_1 \mathbb E_{X_1\sim \mu_1} 
	\bigg[\log\frac{\pi_1(X_1)}{F_\#\pi_0(X_1)}\bigg].
\end{equation*}
Denoting the pushforward density by $P(x) = (F_\#\pi_0)(x)$, we derive the simplified functional
\begin{equation*}
	\mathcal L_J[P] =  \lambda_0 \int P(x) \log \frac{P(x)}{\pi_1(x)}\D x + \lambda_1 \int \mu_1(x)\log \frac{\pi_1(x)}{P(x)}\D x.
\end{equation*}
Because $\int \mu_1(x) \log \pi_1(x)\D x$ is a constant independent of $P$, we can equivalently identify the functional
\begin{equation}
	\mathcal L_J[P] = \lambda_0 D_{\KL}(P\|\pi_1) + \lambda_1 D_{\KL}(\mu_1\|P).
	\label{eq: LJP}
\end{equation}

When $\lambda_0 = 0$, it is straightforward to observe that the unique global minimizer of $\mathcal L_J[P]$ is $P_*(x) = \mu_1(x)$. Consequently, a flow trained with the forward KL cannot surpass the quality of the empirical samples. Meanwhile, the Jeffreys divergence \eqref{eq: LJP} with $\lambda_0>0$ yields rigorous guarantees regarding the sample quality:
\begin{theorem}
\label{theorem: correct}
Assuming the empirical distribution $\mu_0 = \pi_0$, and weighting parameters satisfy $\lambda_0, \lambda_1 > 0$, the following statements hold true:
\begin{enumerate}
	\setlength{\itemsep}{4pt}
	\item The Jeffreys divergence $\mathcal L_J[P]$ in \eqref{eq: LJP} is strictly convex in the pushforward density $P$, ensuring the existence of a unique global minimizer $P_*$.
	\item The unique minimizer $P_*$ satisfies the divergence bound $D_{\KL}(P_*\|\pi_1 ) \Le D_{\KL}(\mu_1\|\pi_1)$.
\end{enumerate}
\end{theorem}
\mbox{}
\begin{proof}
To prove the first claim, we begin by computing the second-order derivative of $\mathcal L_J[P]$ with respect to $P$:
\begin{equation*}
	\frac{\delta^2 \mathcal L_J}{\delta P^2}(x) = \frac{\lambda_0}{P(x)} + 
	\frac{\lambda_1 \mu_1(x)}{P^2(x)} > 0,
\end{equation*}
which immediately demonstrates that $\mathcal L_J[P]$ is strictly convex with respect to $P$. Therefore, $\mathcal L_J[P]$ admits a unique global minimizer $P_*$. Noting that the first-order functional derivative of $\mathcal L_J[P]$ is given by
\begin{equation*}
	\frac{\partial \mathcal L_J}{\partial P} = \lambda_0 \log \frac{P(x)}{\pi_1(x)}
	- \frac{\lambda_1 \mu_1(x)}{P(x)} + \mathrm{const.},
\end{equation*}
we find that the optimal density $P_*$ is determined by the system of equations
\begin{equation*}
\left\{
\begin{aligned}
	\lambda_0 \log \frac{P_*(x)}{\pi_1(x)}
	- \frac{\lambda_1 \mu_1(x)}{P_*(x)} & = c_*, \\
	\int P_*(x) \D x & = 1,
\end{aligned}
\right.
\end{equation*}
where $c_*\in\mathbb R$ represents a normalization constant. 

For the second claim, the optimality condition directly implies $\mathcal L_J[P_*] \Le \mathcal L_J[\mu_1]$. Substituting this into the alternative expression~\eqref{eq: LJP} gives
\begin{equation*}
\lambda_0 D_{\KL}(P_*\|\pi_1) + \lambda_1 D_{\KL}(\mu_1\|P_*) \Le 
\lambda_0 D_{\KL}(\mu_1\|\pi_1 ).
\end{equation*}
Since $D_{\KL}(\mu_1\|P_*)\Ge 0$, we obtain the bound
\begin{equation*}
	\lambda_0 D_{\KL}(P_*\|\pi_1) \Le 
	\lambda_0 D_{\KL}(\mu_1\|\pi_1 ),
\end{equation*}
yielding the desired inequality.
\end{proof}

Theorem~\ref{theorem: correct} ensures that the optimal pushforward distribution $P_*$ achieves a lower KL divergence than the empirical distribution $\mu_1$, allowing for a closer approximation to the target $\pi_1$. This distillation capability stems from the complementary loss terms: while the forward KL mitigates the mode collapse by anchoring the flow to the empirical modes of $\mu_1$, the reverse KL provides a physics-based regularizer. By penalizing deviations against the true potential $U_1$, the reverse KL refines the pushforward density and robustly corrects the bias from the reference data $\mu_1$.

Moreover, the Jeffreys divergence guarantees bounds on the likelihood ratio. Let $P$ and $Q$ be two probability measures on $\mathbb R^d$ with strictly positive densities $P(x)$ and $Q(x)$. We define the $\delta$-instability set $A_\delta$ as the region where the log-likelihood ratio exceeds a threshold $\delta$:
\begin{equation}
	A_\delta = \bigg\{x\in\mathbb R^d:
	\bigg|\log \frac{P(x)}{Q(x)}\bigg| \Ge \delta
	\bigg\}.
\end{equation}
The following lemma establishes that minimizing the Jeffreys divergence bounds the probability of $A_\delta$. For related probability bounds, see the density ratio estimation~\cite{sugiyama2012density} and the Bretagnolle--Huber inequality~\cite{tsybakov2008nonparametric}.
\begin{lemma}
	\label{lemma: concentration}
	Assume that there exists $\epsilon>0$ such that the distributions $P$ and $Q$ satisfy $D_{\KL}(P\|Q) \Le \epsilon$ and $D_{\KL}(Q\|P) \Le \epsilon$. Then the probability of $A_\delta$ is bounded by
	\begin{equation}
		\max\big(P(A_\delta),Q(A_\delta)\big) \Le \frac{2\ep}{\delta(1-e^{-\delta})}.
		\label{eq: mPQb}
	\end{equation}
\end{lemma}
\begin{proof}
By definition, the KL divergence bounds yield
\begin{equation}
\int_{\mathbb R^d} (P(x) - Q(x))\log\frac{P(x)}{Q(x)} \D x \Le 2\epsilon.
\label{eq: PQep}
\end{equation}
Note that the integrand in \eqref{eq: PQep} is non-negative. We partition the instability set $A_\delta$ into
\begin{equation*}
	S_+ = \{P(x) \Ge e^\delta Q(x)\}, \quad S_- = \{Q(x) \Ge e^\delta P(x)\},
\end{equation*}
and define $I_+$ and $I_-$ to be their respective contributions to the integral in \eqref{eq: PQep}. Then \eqref{eq: PQep} implies $I_++I_-\Le 2\epsilon$.

On $S_+$, we have $\log\frac{P(x)}{Q(x)} \Ge \delta$ and
\begin{equation*}
	I_+ \Ge \int_{S_+} P(x)(1-e^{-\delta})\delta \D x = \delta(1-e^{-\delta})P(S_+).
\end{equation*}
Since $Q(x) \Le P(x)$ on $S_+$, we find
\begin{equation}
	Q(S_+) \Le P(S_+) \Le \frac{I_+}{\delta(1-e^{-\delta})}.
	\label{eq: QPd}
\end{equation}
By symmetry, on $S_-$ we obtain
\begin{equation}
	P(S_-) \Le Q(S_-) \Le \frac{I_-}{\delta(1-e^{-\delta})}.
	\label{eq: PQd}
\end{equation}
Since $I_++I_-\Le 2\ep$, adding \eqref{eq: QPd} and \eqref{eq: PQd} yields \eqref{eq: mPQb}.
\end{proof}

Applying Lemma~\ref{lemma: concentration} to our context where $P = F_\# \pi_0$ and $Q = \pi_1$ yields our main theoretical result for the pushforward distribution $F_\#\pi_0$.

\begin{theorem}
\label{theorem: concentration}
Given the distributions $\pi_0$ and $\pi_1$,
suppose $F$ satisfies $D_{\KL}(F_\#\pi_0\|\pi_1) \Le \epsilon$ and $D_{\KL}(\pi_1\|F_\#\pi_0) \Le \epsilon$. For any $\delta > 0$, the probability that the log-likelihood ratio deviates more than $\delta$ is bounded by:
\begin{align}
	\mathbb{P}_{X\sim F_\#\pi_0} \bigg( \bigg| \log \frac{F_\#\pi_0(X)}{\pi_1(X)} \bigg| & \Ge \delta \bigg) 
	\Le \frac{2\epsilon}{\delta(1 - e^{-\delta})},	\\
 \mathbb{P}_{Y\sim \pi_1} \bigg( \bigg| \log \frac{F_\#\pi_0(Y)}{\pi_1(Y)} \bigg| & \Ge \delta \bigg) \Le \frac{2\epsilon}{\delta(1 - e^{-\delta})}.
\end{align}
\end{theorem}
Theorem~\ref{theorem: concentration} guarantees that with high probability, the density ratio of generated distribution $F_\#\pi_0$ and the target $\pi_1$ is bounded. Consequently, this prevents mode collapse (where $\pi_1 \gg F_\#\pi_0$) and spurious mode generation (where $F_\#\pi_0 \gg \pi_1$) as the training error $\epsilon \to 0$.

\mbox{}

\subsection{Unbiased Estimation via Importance Sampling}

In practice, the trained flow $F$ rarely yields $F_\#\pi_0 = \pi_1$ exactly. Nevertheless, unbiased samples from $\pi_1$ can be recovered via importance sampling, provided that $F_\#\pi_0$ covers the entire support of $\pi_1$ without mode collapse.

To construct the importance weights, we first recall the change-of-variable formula for the pushforward density
\begin{equation*}
	(F_\#\pi_0)(y) = \frac{\pi_0(x)}{|\det \nabla F(x)|}, \quad y =F(x).
\end{equation*}
Consequently, we define the unnormalized importance weight $w(y)$ as the likelihood ratio between the target and the generated distributions:
\begin{equation}
	w(y) = \frac{\pi_1(y)}{(F_\#\pi_0)(y)} \propto e^{- U_1(y) + U_0(x) + \log|\det \nabla F(x)|}.
\end{equation}
Notice that $w(y)$ corresponds exactly to the likelihood ratio whose stability is rigorously bounded in Theorem~\ref{theorem: concentration}. If the pushforward density $F_\#\pi_0$ misses regions where $\pi_1$ has significant mass, the corresponding weights $w(y)$ inevitably diverge or exhibit unbounded variance.

For notational convenience, we abstract the single-step Jeffreys Flow operation as:
\begin{equation*}
(F, w) = \mathtt{JeffreysFlow}(\mu_0 , \mu_1 ; U_0 , U_1 ).
\end{equation*}
This notation encapsulates the entire procedure: training the transport map $F$ using the empirical
samples $\mu_0$, $\mu_1$ and potential functions $U_0$, $U_1$, and subsequently computing the importance weights $w$. In particular, we evaluate the flow's quality using the Conditional Effective Sample Size (CESS):
\begin{equation}
	\CESS(F,w) = \frac{\big(\int w(y) \mu_0(F^{-1}(y))\D y\big)^2}{\int w^2(y) \mu_0(F^{-1}(y))\D y},
\end{equation}
which lies in $[0,1]$ due to Cauchy--Schwarz inequality. CESS also serves as the standard metric in SMC~\cite{doucet2001introduction} for evaluating sample degeneracy, where a value approaching $1$ indicates high sample quality.

For a discrete empirical distribution
\begin{equation*}
	\mu_0(x) = \sum_{i=1}^N \alpha_i \delta(x-x_i),\quad 
	\text{where} ~
	\sum_{i=1}^N \alpha_i = 1,
\end{equation*}
its Effective Sample Size (ESS) is defined as
\begin{equation}
	\ESS(\mu_0) = \frac{\big(\sum_{i=1}^N \alpha_i\big)^2}{N\sum_{i=1}^N \alpha_i^2} \Le 1.
\end{equation}
In the specific case where $\mu_0$ has uniform weights
\begin{equation*}
	\mu_0(x) = \frac{1}{N} \sum_{i=1}^N \delta(x-x_i),
\end{equation*}
and $\mu_1 = F_\# \mu_0$, the CESS of the flow $F$ simplifies to
\begin{equation*}
	\CESS(F,w) = \frac{\big(\sum_{i=1}^N w(y_i)\big)^2}{N\sum_{i=1}^N w^2(y_i)},\quad \text{with} ~ y_i = F(x_i),
\end{equation*}
recovering $\ESS(F_\# \mu_0)$.

\subsection{Choice of Hyperparameters \(\lambda_0, \lambda_1\)}

In the Jeffreys divergence in \eqref{eq: Jdiv}, we choose the parameters $\lambda_0$ and $\lambda_1$ by the variance of distributions:
\begin{equation}
	\lambda_0 = \theta \Var(\mu_0), \quad \lambda_1 = (1 - \theta) \Var(\mu_1),
	\label{eq: lambda 01}
\end{equation}
where $\Var(\cdot)$ denotes the variance (trace of covariance matrix), and $\theta \in [0, 1]$ is a balancing hyperparameter.

While the algorithm is generally robust to the choice of $\theta$, extreme values degrade performance: $\theta$ close to $1$ reduces to the reverse KL and can induce mode collapse, while $\theta$ close to $0$ reduces to the forward KL and produces diffused samples with low ESS (see Section~\ref{subsection: single step transport}). In practice, we select $\theta$ by monitoring the CESS; alternatively, using the R\'enyi divergence to achieve stronger geometric constraints for mode-covering.

\section{Sequential Distillation from Parallel Tempering}
\label{section: sequential distillation}

\subsection{Reference Samples from Parallel Tempering}

Given a base distribution $\pi_0(x) \propto e^{-U_0(x)}$ such as a Gaussian or uniform distribution, we aim to sample from the multi-modal target $\pi_M(x) \propto e^{-U_M(x)}$. Directly training a deterministic flow $F$ to map $\pi_0$ to $\pi_M$ is often infeasible, as it requires $F$ to learn complex deformations. Moreover, constrained by the topological structure of diffeomorphisms, deterministic flows are prone to establishing artificial bridges between well-separated local modes.

To prevent this, Boltzmann generators \cite{noe2019boltzmann} employ a temperature ladder to bridge the base and target distributions. Let $M$ be the total number of transitions, and consider the interpolated potential sequence
\begin{equation}
	U_k(x) = \lambda_k U_M(x) + (1-\lambda_k) U_0(x), \quad 0\Le k\Le M,
\end{equation}
where the pacing parameters $\{\lambda_k\}_{k=0}^M$ satisfy
\begin{equation*}
	0 = \lambda_0 < \lambda_1 < \cdots < \lambda_M = 1.
\end{equation*}
This sequence induces $M-1$ intermediate distributions $\pi_k(x) \propto e^{-U_k(x)}$ from high to low temperatures, and decomposes the generation process into $M$ localized flow transformations $F_k$, each satisfying $(F_k)_\# \pi_{k-1} = \pi_k$. To further avoid the spurious mode connections, SNF \cite{wu2020stochastic} interleaves stochastic diffusion steps at each flow layer.

In contrast to standard Boltzmann generators based on reverse KL divergence, the Jeffreys Flow requires reference samples for each intermediate distribution $\pi_k$: training $F_k$ requires approximate samples from both $\pi_{k-1}$ and $\pi_k$. In this case, PT \cite{syed2022non} emerges as a natural choice to generate the required training data.

Over the family of distributions $\{\pi_k\}_{k=0}^M$, PT simulates $M+1$ overdamped Langevin dynamics in parallel:
\begin{equation*}
	\D X_k(t) = -\nabla U_k(X_k(t))\D t + \sqrt{2}\D B_k(t),
	\quad 0\Le k\Le M,
\end{equation*}
where $\{B_k(t)\}_{k=0}^M$ are independent Brownian motions in $\mathbb R^d$. 
Although each process $X_k(t)$ admits $\pi_k(x) \propto e^{-U_k(x)}$ as its invariant distribution,
sampling $\pi_k$ for large $k$ suffers from slow mixing times due to high energy barriers.
To ensure global ergodicity, PT proposes configuration exchanges between adjacent replicas $X_{k-1}(t) = x_{k-1}$ and $X_k(t) = x_k$. The energy difference associated with this proposed swap is given by
\begin{equation*}
	\Delta H_k = U_{k-1}(x_k) + U_k(x_{k-1}) - U_{k-1}(x_{k-1}) - U_k(x_k).
\end{equation*}
To preserve the detailed balance condition, this proposed exchange is accepted with the Metropolis--Hastings probability $\min\{1, e^{-\Delta H_k}\}$.

In general, PT simulations require all $M+1$ processes $\{X_k(t)\}_{k=0}^M$ simultaneously, and high-accuracy samples usually demand a very small time step. In long-time simulations, this can impose a significant CPU and memory burden. However, in the Jeffreys Flow, only low-fidelity samples that cover all modes are needed to be generated by PT. Once the flow models $\{F_k\}_{k=1}^M$ are trained on these samples, we can utilize the flows to instantly generate a large number of accurate samples from the target distribution $\pi_M$. Notably, monitoring the sampling quality in PT is straightforward: standard practice dictates that an average acceptance rate of 20\%--40\% is typically sufficient to ensure global ergodicity \cite{syed2022non}.

\subsection{Sequential Flow Training Framework}

Next, our goal is to learn a sequence of flows $\{F_k\}_{k=1}^M$ such that $(F_k)_\# \pi_{k-1} = \pi_k$. Assuming PT generates an approximate empirical distribution $\mu_k$ for each $\pi_k$, a natural approach is to use $\mu_{k-1}$ and $\mu_k$ to train the flow $F_k$:
\begin{equation*}
	(F_k, w_k) = \mathtt{JeffreysFlow}(\mu_{k-1} , \mu_k ; U_{k-1} , U_k ).
\end{equation*}
However, because the empirical distributions $\mu_{k-1}$ and $\mu_k$ are inherently biased, the resulting $F_k$ may deviate from the true minimizer of the Jeffreys divergence, leading to a low CESS in the distillation steps.

To address this issue, we employ a more robust sequential training strategy, as described in Fig.~\ref{figure: architecture}. The distributions $\{\mu_k\}_{k=0}^M$ and $\{\nu_k\}_{k=0}^M$ represent the training and validation samples, respectively. We initialize a large ensemble $\nu_0$
from the base distribution $\pi_0\propto e^{-U_0}$. At each step $k=1,\dots,M$, we construct a refined empirical distribution $\mu_{k-1}'$ by resampling from the large ensemble $\nu_{k-1}$ and train the flow $F_k$ via
\begin{equation*}
	(F_k, w_k) = \mathtt{JeffreysFlow}(\mu_{k-1}' , \mu_k ; U_{k-1} , U_k ).
\end{equation*}
After training, we push forward $\nu_{k-1}$ using $(F_k, w_k)$ to obtain the updated distribution $\nu_k$. Because of the importance weights, each $\nu_k$ provides an unbiased estimate of $\pi_k$. Thus, the final output $\nu_M$ accurately approximates the target distribution $\pi_M$.

\begin{algorithm*}
\label{algorithm}
\setstretch{1.15}
\caption{Jeffreys Flow: Sequential Distillation from Parallel Tempering}
\KwIn{base potential $U_0$, target potential $U_M$, interpolating parameters $\{\lambda_k\}_{k=0}^M$, reference samples $\{\mu_k\}_{k=0}^M$}
\KwOut{trained flow models $\{F_k\}_{k=0}^M$, high-fidelity output samples $\{\nu_k\}_{k=0}^M$}
\tcp{A. Initialization}
Generate samples $\nu_0$ from the base distribution $\pi_0\propto e^{-U_0}$ with uniform weights.

\tcp{B. Flow Distillation}
\For{$k=1,\dots,M$}{
	\tcp{I. Data Preparation}
	Set the intermediate potential $U_k(x) = \lambda_k U_M(x) + (1-\lambda_k) U_0(x)$. \\
	Resample the ensemble $\mu_{k-1}'$ from the source distribution $\nu_{k-1}$. \\
	\emph{(Optional)} Apply Monte Carlo rejuvenation on $\mu_{k-1}'$ to update the positions. \\
	\tcp{II. Flow Optimization}
	Compute the balancing parameters $\lambda_0 = \Var(\mu_{k-1}')$ and $\lambda_1 = \Var(\mu_k)$. \\
	Train the flow $F_k$ and compute the weights $w_k$ by minimizing the Jeffreys divergence defined in \eqref{eq: rKL}--\eqref{eq: Jdiv}:
	\begin{equation*}
		(F_k, w_k) = \mathtt{JeffreysFlow}(\mu_{k-1}' , \mu_k ; U_{k-1} , U_k ).
	\end{equation*}
	
	\tcp{III. Propagation and Reweighting}
	Compute the pushforward and reweighted distribution $\nu_k$:
	\begin{equation*}
		\nu_k = (F_k, w_k)_\# \nu_{k-1}.
	\end{equation*}
	
	\tcp{IV. Adaptive Resampling for low ESS}
	\If{$\ESS(\nu_k)<50\%$}{
		Resample the distribution $\nu_k$ from itself to restore uniform weights.
	}
	\emph{(Optional)} Apply Monte Carlo rejuvenation on $\nu_k$ to update the positions.
}

\tcp{C. Output Results}
Output the trained flow models $\{F_k\}_{k=0}^M$ and the samples $\{\nu_k\}_{k=0}^M$.
\end{algorithm*}

In practice, the sample size of the output ensemble $\nu_k$ can and should be significantly larger than that of the training ensemble $\mu_k$. On the one hand, we require $\nu_k$ to achieve higher statistical accuracy. On the other hand, evaluating the flow pushforward $F_k$ on $\nu_k$ during inference is computationally much cheaper than performing repeated backpropagation on $\mu_k$ during training. Thus, it is computationally reasonable to use this strategy.

We present the full pipeline of the Jeffreys Flow in Algorithm~\ref{algorithm}. Beyond the core distillation procedure, two optional enhancements are included. First, following the Stochastic Normalizing Flow \cite{wu2020stochastic}, rejuvenation steps can be applied at each step $k$: a few iterations of the Metropolis-Adjusted Langevin Algorithm (MALA) on $\nu_k$ with respect to $\pi_k$ help eliminate artificial bridges connecting different modes (see Section~\ref{subsection: numerical reSGLD}). Second, an adaptive resampling strategy forces the weighted distribution $\nu_k$ to be resampled whenever its ESS drops below $50\%$, preventing weight degeneracy.

\section{Applications in Replica Exchange SGLD}
\label{section: reSGLD}

Many sampling tasks, particularly in Bayesian inference, involve a potential energy function with a finite-sum structure:
\begin{equation}
	U_M(x) = \frac{1}{L} \sum_{i=1}^L V^i(x),
\end{equation}
where $\{V^i(x)\}_{i=1}^L$ is a collection of component potentials in $\mathbb{R}^d$. When the target distribution $\pi_M \propto e^{-U_M}$ is multi-modal, designing an efficient and accurate sampling algorithm poses a significant challenge.

The Replica Exchange Stochastic Gradient Langevin Dynamics (reSGLD)~\cite{deng2020accelerating} offers a scalable framework to sample such distributions. To reduce the computational cost of evaluating full gradients, reSGLD employs a mini-batch potential for the PT simulation:
\begin{equation}
	U_k^{\mathcal B}(x) = (1-\lambda_k) U_0(x) + \frac{\lambda_k}{|\mathcal B|} \sum_{i\in\mathcal B} V^i(x),
\end{equation}
where the subset $\mathcal B \subset \{1,\dots,L\}$ is shared across all temperature indices $k=0,1,\dots,M$. Consequently, the Langevin dynamics is replaced by its stochastic variant:
\begin{equation*}
    X_k(t + h) = X_k(t) - h \nabla U_k^{\mathcal B}(X_k(t)) + \sqrt{2h}(B_k(t + h) - B_k(t)),
\end{equation*}
where $\{B_k(t)\}_{k=0}^M$ are independent Brownian motions in $\mathbb{R}^d$. Since the stochastic gradient is unbiased, this numerical scheme is exact in the continuous-time limit $h \to 0$.

However, computing the swapping rate between adjacent replicas presents a core challenge. Using the mini-batch approximation, the stochastic energy difference for swapping configurations $x_{k-1}$ and $x_k$ at pacing parameters $\lambda_{k-1}$ and $\lambda_k$ is
\begin{equation*}
	\Delta H_k^{\mathcal B} = U_{k-1}^{\mathcal B}(x_k) + U_k^{\mathcal B}(x_{k-1}) - U_{k-1}^{\mathcal B}(x_{k-1}) - U_k^{\mathcal B}(x_k),
\end{equation*}
yielding the acceptance probability
\begin{equation}
\alpha_k^{\mathcal B} = \min\big\{1, e^{-\Delta H_k^{\mathcal B}}\big\}.
\label{eq: akB}
\end{equation}
Due to the nonlinearity of the exponential function, \eqref{eq: akB} yields a biased estimator of the true swapping probability, even if $\Delta H_k^{\mathcal B}$ itself remains unbiased. Variance correction techniques \cite{deng2020accelerating} can mitigate this bias, but we do not employ them here for simplicity.

Although reSGLD can generate the reference sample $\{\mu_k\}_{k=0}^M$ analogous to standard PT, it inherently sacrifices theoretical exactness in computing the acceptance probability, even in the limit $h \to 0$. By contrast, the Jeffreys Flow effectively circumvents this limitation---the Jeffreys divergence for training the map $F_k$, defined in \eqref{eq: rKL}--\eqref{eq: Jdiv}, depends linearly on the potential functions $U_{k-1}$ and $U_k$. Consequently, Jeffreys Flow theoretically guarantees the exact transport map $F_k$, requiring however additional training epochs to average out the variance.

After training $F_k$, computing the importance weights using the full exact potential is highly affordable, as it requires only a single evaluation across the ensemble $\nu_{k-1}$. In contrast, reSGLD demands frequent replica swapping to maintain ergodicity, hence a mini-batch approximation to the energy difference remains essential for efficiency. Furthermore, the Monte Carlo rejuvenation steps in Algorithm~\ref{algorithm} can be implemented via a Stochastic Variance-Reduced Gradient (SVRG) \cite{johnson2013accelerating,reddi2016stochastic} strategy. By initially evaluating the full gradient for the entire ensemble $\nu_k$, this variance reduction technique yields highly accurate gradients, enabling us to safely simulate the Langevin dynamics without Metropolis rejection steps.

Even when evaluating the full potential is computationally prohibitive, we can still perform accurate importance sampling via a residual matching technique. Specifically, we train auxiliary neural networks $\phi$ and $\psi$ by minimizing the $L^2$ loss:
\begin{align*}
    \mathcal{L}_2[\phi, \psi] & = \mathbb{E}_{X_0 \sim \nu_{k-1}} \Big[ \big( U_k^{\mathcal{B}}(F_k(X_0)) - U_{k-1}^{\mathcal{B}}(X_0) \\
    & \hspace{2cm} + \phi(X_0) - \psi(F_k(X_0)) \big)^2 \Big],
\end{align*}
where $U_k^{\mathcal{B}}$ and $U_{k-1}^{\mathcal{B}}$ denote the strictly unbiased mini-batch approximations of the potentials. Upon convergence, we utilize $\psi(y) - \phi(x) - \log |\det \nabla F_k(x)|$ as a surrogate for the intractable log-importance weights $\log w_k(y)$. This surrogate achieves high accuracy because the $L^2$ residual loss provides a significantly finer resolution than the standard KL divergence.

\section{Applications in Path Integral Monte Carlo}
\label{section: PIMC}

\subsection{Path Integral Formulation}

In Path Integral Monte Carlo (PIMC), the primary objective is to compute quantum thermal averages by sampling a quantum canonical ensemble. Under the path integral formulation~\cite{reed1972methods}, this ensemble is mathematically equivalent to a continuous probability distribution over an infinite-dimensional space of imaginary-time paths.

To formalize this, consider a quantum system in $\mathbb{R}^d$ governed by the Hamiltonian operator:
\begin{equation}
    \hat{H} = -\frac{1}{2} \Delta + V(x),
\end{equation}
where $\Delta$ is the Laplacian and $V: \mathbb{R}^d \to \mathbb{R}$ is the potential energy function. For a system at inverse temperature $\beta$, the density matrix of the thermal Boltzmann ensemble $e^{-\beta \hat{H}}$ exactly corresponds to a probability measure on the space of closed continuous paths $x: [0, \beta] \to \mathbb{R}^d$. This measure is governed by the Euclidean action functional:
\begin{equation}
    U(x(\cdot)) = \int_0^\beta \left( \frac{1}{2} |\dot{x}(t)|^2 + V(x(t)) \right) \D t.
\end{equation}

Here, the kinetic energy term $\frac{1}{2} |\dot{x}(t)|^2$ induces a Gaussian reference measure tying the path into a closed ring, while the potential $V(x(t))$ modifies the local density. The target distribution is thus formally given by:
\begin{equation*}
    \pi(x(\cdot)) \propto e^{-U(x(\cdot))}.
\end{equation*}
Generating samples from this infinite-dimensional measure provides a direct and exact framework for evaluating macroscopic thermodynamic properties.

To computationally sample the infinite-dimensional measure $\pi(x(\cdot))$, we construct a finite-dimensional approximation. By discretizing the imaginary-time interval $[0, \beta]$ into $N$ equal segments, we represent the continuous path as a discrete sequence of beads $x_1, \dots, x_N \in \mathbb{R}^d$ subject to the periodic boundary condition $x_{N+1} = x_1$. The Euclidean action functional is thus approximated by the discrete ring-polymer potential:
\begin{equation}
    U_N(x_{1:N}) = \frac{N}{2\beta} \sum_{j=1}^N |x_j - x_{j+1}|^2 + \frac{\beta}{N} \sum_{j=1}^N V(x_j),
    \label{eq: UNx}
\end{equation}
and the discretized target distribution is now:
\begin{equation*}
	\pi_N(x_{1:N})\propto \exp\big(-U_N(x_{1:N})\big).
\end{equation*}
To decouple the stiff harmonic interactions between adjacent beads, we assume $N$ is even and apply a discrete Fourier transform. This maps the Cartesian coordinates $\{x_k\}_{k=1}^N$ into the normal mode coordinates $\{\xi_k\}_{k=0}^{N-1}$:
\begin{equation}
    \xi_k = \frac{\beta}{N} \sum_{j=1}^N x_j c_{j,k}, \quad k = 0, 1, \dots, N-1.
\end{equation}
Here, the orthogonal transformation matrix elements $c_{j,k}$ are explicitly given by:
\begin{align*}
    c_{j,0} &= \frac{1}{\sqrt{\beta}}, \quad c_{j,N-1} = \frac{(-1)^j}{\sqrt{\beta}}, \\
    c_{j,2k-1} &= \sqrt{\frac{2}{\beta}} \sin\left(\frac{2\pi k j}{N}\right), \quad
    c_{j,2k} = \sqrt{\frac{2}{\beta}} \cos\left(\frac{2\pi k j}{N}\right), \\
    & \text{for } k = 1, \dots, \frac{N}{2} - 1;
\end{align*}
and they rigorously satisfy orthogonality:
\begin{equation*}
    \frac{\beta}{N} \sum_{j=1}^N c_{j,k} c_{j,k'} = \delta_{k,k'}.
\end{equation*}

By transitioning to this normal mode basis, the stiff harmonic interaction term is exactly diagonalized:
\begin{equation}
    \frac{N}{2\beta} \sum_{j=1}^N |x_j - x_{j+1}|^2 = \frac{1}{2} \sum_{k=0}^{N-1} \omega_k^2 |\xi_k|^2,
\end{equation}
where the intrinsic ring-polymer frequencies $\{\omega_k\}_{k=0}^{N-1}$ governing each Fourier mode are analytically defined by:
\begin{equation*}
    \omega_0 = 0, \quad \omega_{2k-1} = \omega_{2k} = \frac{2N}{\beta} \sin\left(\frac{k\pi}{N}\right).
\end{equation*}
Consequently, $U_N(x_{1:N})$ in \eqref{eq: UNx} can be equivalently expressed in the normal mode coordinates as:
\begin{equation*}
    U_N(\xi_{0:N-1}) = \frac{1}{2} \sum_{k=0}^{N-1} \omega_k^2 |\xi_k|^2 + \frac{\beta}{N} \sum_{j=1}^N V\bigg( \sum_{k=0}^{N-1} \xi_k c_{j,k} \bigg).
\end{equation*}

By sampling these finite normal modes $\xi_{0:N-1}$, we systematically construct a continuous function that serves as an accurate approximation to the infinite-dimensional statistical measure. Due to the symmetric nature of the Lie--Trotter splitting applied in the potential discretization, this Fourier representation achieves a numerical convergence bounded by $O(1/N^2)$ \cite{ye2023optimal}.

\subsection{Physics-Informed Mode Truncation}

A fundamental challenge in this framework is that achieving high-accuracy simulations demands a very large number of modes $N$. Simultaneously, standard PT necessitates multiple replicas across the temperature ladder. For high-dimensional target potentials, this multiplicative expansion of the state space results in severe memory overhead and prohibitive computational costs. This dimensionality bottleneck also identically afflicts flow-based generation, where the cost of training a normalizing flow scales poorly---often exponentially---with the number of modes $N$.

Fortunately, this limitation can be effectively mitigated through a physics-informed mode truncation technique. Rather than training a flow over all $N$ modes, we minimize the Jeffreys divergence exclusively for the $N_0$ low-frequency modes $\{\xi_k\}_{k=0}^{N_0-1}$, governed by $U_{N_0}(\xi_{0:N_0-1})$. For the high-frequency modes, the flow simply applies an identity mapping. Because the low-frequency modes govern the macroscopic topology of the ring polymer, this truncated flow is sufficient to capture the essential geometry of the target distribution. The residual discretization errors are subsequently rigorously corrected via importance sampling using the full potential $U_N(\xi_{1:N})$.

Furthermore, the PT simulation can be implemented highly efficiently. We can approximate the target potential $U_N(\xi_{0:N-1})$ via its purely classical approximation:
\begin{equation*}
	U_N(\xi_{0:N-1}) \approx \frac{1}{2} \sum_{k=1}^{N-1} \omega_k^2 |\xi_k|^2 + V\bigg(\frac{\xi_0}{\sqrt{\beta}}\bigg),
\end{equation*}
where all modes $\{\xi_k\}_{k=0}^{N-1}$ are strictly decoupled. Consequently, we only need to generate PT samples for the classical potential $V(x)$ in $\mathbb{R}^d$, while the remaining modes explicitly follow independent Gaussian distributions.

In practice, we employ the following procedure to deploy Jeffreys Flow on an $N$-mode ring polymer system:
\begin{enumerate}[(i)]
	\setzero
	\item Generate reference PT samples for the classical system $V(x)$ in $\mathbb{R}^d$.
	\item Select a small $N_0$, and train the flows with the truncated potential $U_{N_0}(\xi_{0:N_0-1})$. Then push forward the ensemble with the full potential $U_N(\xi_{0:N-1})$.
	\item Utilize these generated samples to perform a second run of Jeffreys Flow training.
\end{enumerate}
Because this secondary training leverages theoretically unbiased pushforward samples instead of PT samples, the corresponding CESS is significantly elevated.

\mbox{}

\section{Numerical Experiments}
\label{section: numerical experiments}

In this section, we evaluate the Jeffreys Flow on a diverse set of benchmark distributions up to $d=16$. Throughout, the PT temperature ladder is calibrated to maintain a swapping rate of approximately $40\%$ for global ergodicity, and each flow layer is parameterized by the Neural Spline Flow architecture~\cite{durkan2019neural}.

\subsection{Single-Step Transport}
\label{subsection: single step transport}

We evaluate the Jeffreys Flow on a single-step transport problem, which involves mapping a 2D Gaussian or uniform base distribution to multi-modal targets. To assess both mode coverage and approximation accuracy, we compare the Jeffreys divergence against the baselines of forward and reverse KL divergences. As demonstrated in \eqref{eq: lambda 01}, the parameter $\theta$ serves as an interpolation weight between these two constituent divergences.

\subsubsection*{Four Potentials in Full Space}

We evaluate our method on four 2D potential  functions $U(x_1, x_2)$: Three Well (TW), Himmelblau (HB), Annulus (AN), and Multiple Well (MW). Their explicit forms are:
\begin{enumerate}[(i)]
	\setzero
	\item TW: $3(x_1-1)^2 + 3(x_2-1)^2 + 3\sin(x_1+2x_2)$ 
	\item HB: $0.2(x_1^2 + x_2 - 11)^2 + 0.2(x_1 + x_2^2 - 7)^2$
	\item AN: $10(r^6 - 8r^4 + 16r^2 + 1)^{\frac13}, \quad r^2=x_1^2+x_2^2$ 
	\item MW: $\frac{1}{4}x_1^2 + \frac{1}{4} x_2^2 + 4 \sin(1.1x_1) \sin(x_2)$ 
\end{enumerate}

To test robustness, we generate $20,000$ noisy reference samples using unconverged PT, intentionally introducing visible artifacts. After training the flows across various $\theta$, we generate $320,000$ new samples (16 times the training set size) to evaluate the ESS and approximation bias.

The results in Fig.~\ref{fig: single_step} demonstrate that pure forward KL ($\theta=0$) covers all modes but produces highly diffused samples with low ESS, while pure reverse KL ($\theta=1$) immediately suffers from catastrophic mode collapse and massive bias. As clearly visualized in the generated sample distributions, the Jeffreys Flow ($0 < \theta < 1$) explicitly resolves these issues: it leverages forward KL to guarantee global mode coverage while utilizing reverse KL to enforce target-seeking precision. Consequently, intermediate values like $\theta=0.5$ rapidly correct the noisy initial samples, yielding highly accurate, low-bias distributions and significantly elevated ESS across all test potentials.

\subsubsection*{Periodic Well Potential}

Furthermore, we test the robust scalability of the Jeffreys Flow on a 2D Periodic Well potential (PW):
\begin{equation*}
	\mathrm{PW:}~4 \sin(2x_1 )(\sin(2x_2))^{\frac75}, \quad 
	x_1, x_2 \in [-\pi, \pi],
\end{equation*}
where the power is implemented as $\mathrm{sgn}(\cdot)|\cdot|^{1.4}$ for differentiability. Since the domain is a torus, we train a circular spline flow using $20,000$ reference samples generated from PT. For simplicity, the balancing parameter is fixed at $\theta=0.5$. We then generate sample ensembles of sizes $N \in \{4^8, 4^9, \dots, 4^{12}\}$ and measure the ESS and $L^2$ approximation bias against the analytically tractable ground truth.

As visualized in Fig.~\ref{fig: single_step_periodic}, the trained Jeffreys Flow perfectly covers all eight symmetric potential wells. The model achieves a remarkably stable ESS near $100\%$ across all generation scales, confirming that the flow correctly captures the exact target geometry. Most importantly, the $L^2$ bias exhibits a strict log-linear decay with respect to $N$. This proves that the optimally trained Jeffreys Flow guarantees standard Monte Carlo error convergence without introducing any asymptotic bias floor.

\subsection{Multi-Step Boltzmann Generator}

We apply the Jeffreys Flow to train Boltzmann generators for higher-dimensional distributions. These models gradually transport a base Gaussian or uniform distribution toward a complex multi-modal target, while the balancing parameter $\theta$ is held fixed. Sample quality is rigorously evaluated using the ESS at each intermediate step, alongside structural visualizations of the generated target distributions.

\subsubsection*{3D Gaussian Mixture Model}

The target is a six-component Gaussian mixture in an octahedral geometry with potential
\begin{equation*}
	U(x)=-\log \bigg( \frac{1}{6} \sum_{k=1}^6 \exp \bigg( -\frac{1}{2} (x - \mu_k)^\T \Sigma_k^{-1} (x - \mu_k) \bigg) \bigg),
\end{equation*}
where the means $\mu_k \in \{(\pm 6, 0, 0), (0, \pm 6, 0), (0, 0, \pm 6)\}$ and covariances $\Sigma_k$ define distinct anisotropic modes.

We train the Jeffreys Flow ($\theta = 0.75$) via 3 intermediate steps using $50,000$ PT reference samples. As shown in Fig.~\ref{fig: gm_results}, the distilled flow successfully transports the Gaussian base to the multi-modal target. The consistently high ESS over $80\%$ confirms the flow accurately covers all six modes, and the resulting bias of the Jeffreys Flow is much smaller than PT.

\subsubsection*{4D Rosenbrock}

The 4D Rosenbrock function features a twisted, narrow valley with potential
\begin{equation*}
	U(x)=6\sum_{i=1}^3 \Big[ 100(x_{i+1} - x_i^2)^2 + (1 - x_i)^2 \Big],
\end{equation*}
introducing strong nonlinear correlations between adjacent variables, dominated by the curve $x_{i+1} = x_i^2$.

We train the Jeffreys Flow over $11$ steps using $50,000$ PT reference samples. As shown in Fig.~\ref{fig: rb_results}, the flow effectively navigates the banana-shaped geometry, a notoriously ill-conditioned structure that is generally difficult for standard normalizing flows to learn. The high CESS confirms that the model precisely captures all intricate variable dependencies.

\subsubsection*{8D Nonlinear Rastrigin}

To evaluate the method in higher dimensions, we construct a highly non-convex potential by composing a Rastrigin function with a nonlinear diffeomorphism. The potential is defined as:
\begin{equation*}
	U(x) = \sum_{i=1}^8 \Big[ 0.5z_i^2 + 12 \cos(2z_i) \Big], ~ z = x + \delta \tanh(xQ).
\end{equation*}
Here, $Q$ is a fixed random orthogonal matrix and $\delta = 0.75$ controls the strength of the nonlinearity. This transformation introduces complex dependencies between all dimensions, making the mode separation geometry extremely difficult to traverse.

For this high-dimensional task, we increase the training sample size to $N = 400{,}000$ and set $\theta=0.5$. The algorithm maintains a high CESS above $70\%$ throughout all stages, as summarized in Table~\ref{tab: nr_metrics}. The ESS exhibits a sawtooth pattern, decaying during transport and restoring to $100\%$ upon adaptive resampling. Furthermore, Fig.~\ref{fig: nr_results} displays strong structural agreement of the potential energy distribution and the 1D marginal density between Jeffreys Flow and PT, confirming successful global mode coverage.

\subsubsection*{16D Solvated Periodic Grid}

Consider a 16-dimensional system modeling a particle interacting with a harmonic solvent bath upon a periodic substrate. The potential on $\mathbb{T}^2 \times \mathbb{R}^{14}$ is
\begin{align*}
	U(x) & = A \big[ \cos(4x_1) + \cos(4x_2) \big] \\
	& + \frac{\kappa}{2} \sum_{k=3}^{16} (x_k - \alpha_k \sin(x_1) \sin(x_2))^2,
\end{align*}
where $x_1, x_2 \in [-\pi, \pi)$ are periodic particle coordinates, and $x_3, \ldots, x_{16}$ are full-space solvent variables. The coupling coefficients $\alpha_k$ are uniformly spaced in $[2, 4]$, with barrier height $A=2$ and solvent stiffness $\kappa=30$. Theoretically, integrating out the bath variables renders $x_1$ and $x_2$ strictly independent.

We train the Jeffreys Flow ($\theta = 0.75$, R\'enyi $\alpha=1.5$), utilizing a composite mapping of circular and rational quadratic spline flows to accommodate the mixed manifold domain. As shown in Fig.~\ref{fig: sl_joint}, the PT reference suffers from severe spurious diagonal correlations, failing to break the kinetic barriers induced by the solvent. In contrast, Jeffreys Flow perfectly recovers the theoretical independent checkerboard structure, effectively uncoupling the periodic variables and correcting the massive bias in the training data.

To further quantify this structural correction, we reconstruct the Potential of Mean Force (PMF) along $x_1$. As depicted in Fig.~\ref{fig: sl_pmf}, the PT sampler underestimates the free energy barriers due to broken ergodicity. Conversely, the distilled Jeffreys Flow demonstrates striking agreement with the analytical truth, accurately reproducing both the barrier heights and well topologies.

\subsection{Application in Replica Exchange SGLD}
\label{subsection: numerical reSGLD}

\subsubsection*{2D Gaussian Mixture Model}

The target distribution consists of four anisotropic Gaussian components, defined by the potential function
\begin{equation*}
	U(x) = -\log \bigg( \sum_{k=1}^4 w_k \exp \Big( -\frac{1}{2}(x - \mu_k)^\top \Sigma_k^{-1} (x - \mu_k) \Big) \bigg),
\end{equation*}
where $x = (x_1, x_2)^\top \in \mathbb{R}^2$, and the mixture weights are $w = [0.25, 0.30, 0.15, 0.30]$. The component means $\mu_k$ are distributed across the four quadrants to ensure mode separation, while the covariance matrices $\Sigma_k$ provide distinct anisotropic orientations for each well. To emulate the stochastic gradient noise inherently present in SGLD, we inject an auxiliary unbiased random potential:
\begin{equation*}
	\tilde{U}(x) = \tilde{i}_1 \cos(x_1) + \tilde{i}_2 \cos(x_2), \quad \tilde{i}_1, \tilde{i}_2 \sim U\{-4, 4\},
\end{equation*}
such that the total effective potential evaluated at each step is $U(x) + \tilde{U}(x)$.

We first generate $40,000$ rough reference samples via reSGLD across $6$ temperatures, purposefully employing a large discrete step size $h = 10^{-2}$. We then train the Jeffreys Flow with $\theta = 0.5$, and compare enabling versus disabling SVRG for the rejuvenation steps (Sec.~\ref{section: reSGLD}).

As detailed in Tables~\ref{tab: resgld_bias} and \ref{tab: resgld_ess}, Jeffreys Flow successfully drops the massive $L^2$ bias inherent in raw reSGLD chains by roughly an order of magnitude (thereby filtering out the aggressive discretization errors) while independently maintaining extremely high ESS metrics at scale. Furthermore, employing SVRG rejuvenation during the flow rigorously eliminates artificial topological tails, strictly contracting generated samples to their theoretically correct localized minima (Fig.~\ref{fig: resgld_results}).

\subsubsection*{2D Screened Poisson Inverse Problem}

To evaluate the Jeffreys Flow's scalability on computationally expensive, highly nonlinear tasks, we consider a parameter inference problem governed by the 2D Screened Poisson equation:
\begin{equation*}
	-\Delta u(x) + \alpha u(x) = f(x; \Theta), \quad x \in \Omega = [0, 1]^2,
\end{equation*}
subject to homogeneous Dirichlet boundary conditions $u(x) = 0$ on $\partial\Omega$. Here, $u(x)$ is the physical state variable, $\alpha=1$ is the screening parameter, and $f(x; \Theta)$ represents a source field composed of 4 distinct Gaussian peaks:
\begin{equation*}
	f(x; \Theta) = c \sum_{s=1}^4 \exp \bigg( -\frac{\|x - \theta_s\|^2}{2\gamma^2} \bigg),
\end{equation*}
where $\theta_s \in \mathbb{R}^2$ marks the coordinate of the $s$-th source. The full parameter vector to be inferred is the unknown coordinates $\Theta = [\theta_1^\top, \ldots, \theta_4^\top]^\top \in \mathbb{R}^8$. We set amplitude $c = 1.0$ and width $\gamma = 0.1$. The domain $\Omega$ is discretized using a $40 \times 40$ structured grid, translating each forward evaluation into solving a sparse $1600 \times 1600$ linear system.

As shown in Fig.~\ref{fig: poisson_setup}, measurements are collected by an evenly distributed $9 \times 9$ sensor network array $\mathcal{O}$. The observations are subsequently corrupted by independent Gaussian noise applied to the sensors:
\begin{equation*}
	d_j = u(x_j; \Theta_{\text{true}}) + \eta_j, \quad \eta_j \sim \mathcal{N}(0, \sigma_{\text{obs}}^2), \quad \text{for } x_j \in \mathcal{O}.
\end{equation*}
The inverse task quantifies the posterior uncertainty $\pi(\Theta) \propto \exp(-U(\Theta))$, where the highly non-convex potential energy is precisely the data mismatch loss:
\begin{equation*}
	U(\Theta) = \frac{1}{2\sigma_{\text{obs}}^2} \sum_{j=1}^{|\mathcal{O}|} |u(x_j; \Theta) - d_j|^2.
\end{equation*}

We train the Jeffreys Flow using $\theta = 0.8$ based on reSGLD reference samples across $M = 8$ temperature steps. We benchmark three configurations: Exact (using exact PDE simulations for the weight function), Full Potential (using exact PDE evaluations during the SGLD simulation), and Stochastic (relying strictly on a surrogate neural network for both mini-batch simulation and weight training). The evaluated sampling efficiencies are summarized in Table~\ref{tab: poisson_ess}. While all configurations maintain high CESS, the optimal Exact configuration leverages the high-fidelity PDE solver to contract the posterior probability peaks tightly, providing a superior, confident parameter estimation compared to the structurally degraded Stochastic result (Fig.~\ref{fig: poisson_marginals}). Overall, by distilling the geometry into an invertible map, Jeffreys Flow perfectly bypasses the computationally prohibitive PDE-solving steps demanded by traditional MCMC chains within such PDE-constrained environments.

\subsection{Applications in Path Integral Monte Carlo}

To validate the efficacy of Jeffreys Flow in mapping infinite-dimensional function spaces, we apply the framework to a quantum ring-polymer problem in 1D. Consider a quantum particle interacting with an asymmetric double-well potential:
\begin{equation*}
	V(x) = 20(x - 1)^2 (x + 0.9)(x + 1.1).
\end{equation*}
In the path integral formulation, the quantum target distribution is a function space of periodic imaginary-time paths with probability measure $\pi(x(\cdot)) \propto \exp(-U(x(\cdot)))$, where the Euclidean action $U$ intricately couples the continuous path geometry with $V(x)$. 

A crucial advantage of our framework is its training simplicity. We generate the reference samples $\{\mu_k\}_{k=1}^M$ by strictly sampling the computationally cheap classical Boltzmann distribution $\exp(-V(x))$. The Jeffreys Flow ($\theta = 0.75$) learns the structural mapping strictly from these classical samples, systematically distilling it to lift the probability mass exactly into the high-dimensional quantum function space. As shown in Fig.~\ref{fig: pimc_density}, the distilled quantum distribution successfully captures deep tunneling effects and broad spatial delocalization entirely absent in the localized classical training data. The consistently high CESS values perfectly confirm the robustness of the transport map during this extreme cross-dimensional distillation (Fig.~\ref{fig: pimc_cess}).

To rigorously quantify the generalization capability of the learned flow into substantially higher dimensions, we assess the generated ensemble across discrete path integral representations utilizing varying numbers of discrete beads $N$. The transport map is trained strictly on a heavily truncated low-frequency Fourier subset comprising only $N_0 = 8$ modes. To sample at progressively higher resolutions without retraining, we pad the missing high-frequency coefficients with an invariant identity map and correct truncation errors through importance reweighting against the exact full-dimensional potential.

Table~\ref{tab: pimc_bias} documents the resulting $L^2$ bias against exact finite-difference eigenspace decompositions for physical observables. Using solely the static $N_0 = 8$ transport map, the reweighted Jeffreys Flow generates high-fidelity samples uniformly up to $N = 32$. The bias systematically plummets on a strict theoretical $\mathcal{O}(1/N^2)$ algebraic decay trajectory, powerfully confirming that our physics-informed mode truncation guarantees scalable, highly efficient sample generation into theoretically infinite-dimensional limits without incurring exponential parameter overhead.

\section{Conclusion}
\label{section: conclusion}

This paper introduces the Jeffreys Flow, a unified generative framework that bridges the forward and reverse KL divergences via the Jeffreys divergence. By symmetrizing the training objective, the flow achieves bidirectional fidelity: it effectively suppresses mode collapse while maintaining target-seeking precision with rigorous density ratio bounds.

Rather than competing directly with traditional Monte Carlo methods, Jeffreys Flow operates as a robust structural distillation mechanism. As demonstrated across fundamentally diverse test cases---including sequential distillation on ill-conditioned spaces, accelerated noise-filtering within Replica Exchange SGLD, and extreme dimensional scaling via mode truncation in Path Integral Monte Carlo---the trained map uncouples kinetic biases and achieves highly scalable, independent feed-forward generation without prohibitive computational overhead. 

Building on this theoretical and empirical foundation, our future work will focus on extending the Jeffreys Flow to more complicated, high-dimensional, and specific physical problems. Potential applications include large-scale molecular dynamics with explicit solvents, highly non-convex Bayesian inverse problems, and lattice field theories, fully unlocking the framework's capability for advanced rare-event simulations.

\section*{Acknowledgment}
G. Lin would like to thank the support of National Science Foundation (DMS-2533878, DMS-2053746, DMS-2134209, ECCS-2328241, CBET-2347401 and OAC-2311848), and U.S.~Department of Energy (DOE) Office of Science Advanced Scientific Computing Research program DE-SC0023161, the SciDAC LEADS Institute, and DOE–Fusion Energy Science, under grant number: DE-SC0024583.
D. Qi would like to thank the support of ONR Grant N00014-24-1-2192, and NSF Grant DMS-2407361.

The numerical tests are implemented with an NVIDIA RTX 5070 Ti, and the source codes can be found at \url{https://github.com/xuda-ye-math/Jeffreys-Flow}.

\bibliography{references/introduction,references/related_works,references/applications}

@article{johnson2013accelerating,
  title={Accelerating stochastic gradient descent using predictive variance reduction},
  author={Johnson, Rie and Zhang, Tong},
  journal={Advances in neural information processing systems},
  volume={26},
  year={2013}
}

@inproceedings{reddi2016stochastic,
  title={Stochastic variance reduction for nonconvex optimization},
  author={Reddi, Sashank J and Hefny, Ahmed and Sra, Suvrit and Poczos, Barnabas and Smola, Alex},
  booktitle={International conference on machine learning},
  pages={314--323},
  year={2016},
  organization={PMLR}
}

@book{reed1972methods,
  title={Methods of modern mathematical physics},
  author={Reed, Michael and Simon, Barry and Simon, Barry and Simon, Barry},
  volume={1},
  year={1972},
  publisher={Elsevier}
}

@article{ye2023optimal,
  title={Optimal Convergence Rate of Lie-Trotter Approximation for Quantum Thermal Averages},
  author={Ye, Xuda and Zhou, Zhennan},
  journal={arXiv e-prints},
  pages={arXiv--2309},
  year={2023}
}

@book{bucklew2004introduction,
  title={Introduction to rare event simulation},
  author={Bucklew, James Antonio and Bucklew, J},
  volume={5},
  year={2004},
  publisher={Springer}
}

@book{rubino2009rare,
  title={Rare event simulation using Monte Carlo methods},
  author={Rubino, Gerardo and Tuffin, Bruno and others},
  volume={73},
  year={2009},
  publisher={Wiley Online Library}
}

@article{rubino2009introduction,
  title={Introduction to rare event simulation},
  author={Rubino, Gerardo and Tuffin, Bruno},
  journal={Rare event simulation using Monte Carlo methods},
  pages={1--13},
  year={2009},
  publisher={Wiley Online Library}
}

@inproceedings{bouchet2016generalisation,
  title={Generalisation of the Eyring--Kramers transition rate formula to irreversible diffusion processes},
  author={Bouchet, Freddy and Reygner, Julien},
  booktitle={Annales Henri Poincar{\'e}},
  volume={17},
  number={12},
  pages={3499--3532},
  year={2016},
  organization={Springer}
}

@article{lee2022non,
  title={Non-reversible metastable diffusions with Gibbs invariant measure I: Eyring--Kramers formula},
  author={Lee, Jungkyoung and Seo, Insuk},
  journal={Probability Theory and Related Fields},
  volume={182},
  number={3},
  pages={849--903},
  year={2022},
  publisher={Springer}
}

@article{eyring1935activated,
  title={The activated complex in chemical reactions},
  author={Eyring, Henry},
  journal={The Journal of Chemical Physics},
  volume={3},
  number={2},
  pages={107--115},
  year={1935},
  publisher={American Institute of Physics}
}

@article{kramers1940brownian,
  title={Brownian motion in a field of force and the diffusion model of chemical reactions},
  author={Kramers, Hendrik Anthony},
  journal={Physica},
  volume={7},
  number={4},
  pages={284--304},
  year={1940},
  publisher={Elsevier}
}

@article{torrie1977nonphysical,
  title={Nonphysical sampling distributions in Monte Carlo free-energy estimation: Umbrella sampling},
  author={Torrie, Glenn M and Valleau, John P},
  journal={Journal of computational physics},
  volume={23},
  number={2},
  pages={187--199},
  year={1977},
  publisher={Elsevier}
}

@article{virnau2004calculation,
  title={Calculation of free energy through successive umbrella sampling},
  author={Virnau, Peter and M{\"u}ller, Marcus},
  journal={The Journal of chemical physics},
  volume={120},
  number={23},
  pages={10925--10930},
  year={2004},
  publisher={American Institute of Physics}
}

@article{kastner2011umbrella,
  title={Umbrella sampling},
  author={K{\"a}stner, Johannes},
  journal={Wiley Interdisciplinary Reviews: Computational Molecular Science},
  volume={1},
  number={6},
  pages={932--942},
  year={2011},
  publisher={Wiley Online Library}
}

@incollection{van1987simulated,
  title={Simulated annealing},
  author={Van Laarhoven, Peter JM and Aarts, Emile HL},
  booktitle={Simulated annealing: Theory and applications},
  pages={7--15},
  year={1987},
  publisher={Springer}
}

@article{bertsimas1993simulated,
  title={Simulated annealing},
  author={Bertsimas, Dimitris and Tsitsiklis, John},
  journal={Statistical science},
  volume={8},
  number={1},
  pages={10--15},
  year={1993},
  publisher={Institute of Mathematical Statistics}
}

@incollection{nikolaev2010simulated,
  title={Simulated annealing},
  author={Nikolaev, Alexander G and Jacobson, Sheldon H},
  booktitle={Handbook of metaheuristics},
  pages={1--39},
  year={2010},
  publisher={Springer}
}

@incollection{doucet2001introduction,
  title={An introduction to sequential Monte Carlo methods},
  author={Doucet, Arnaud and De Freitas, Nando and Gordon, Neil},
  booktitle={Sequential Monte Carlo methods in practice},
  pages={3--14},
  year={2001},
  publisher={Springer}
}

@article{cappe2007overview,
  title={An overview of existing methods and recent advances in sequential Monte Carlo},
  author={Capp{\'e}, Olivier and Godsill, Simon J and Moulines, Eric},
  journal={Proceedings of the IEEE},
  volume={95},
  number={5},
  pages={899--924},
  year={2007},
  publisher={IEEE}
}

@book{chopin2020introduction,
  title={An introduction to sequential Monte Carlo},
  author={Chopin, Nicolas and Papaspiliopoulos, Omiros and others},
  volume={4},
  year={2020},
  publisher={Springer}
}

@article{wills2023sequential,
  title={Sequential Monte Carlo: a unified review},
  author={Wills, Adrian G and Sch{\"o}n, Thomas B},
  journal={Annual Review of Control, Robotics, and Autonomous Systems},
  volume={6},
  number={1},
  pages={159--182},
  year={2023},
  publisher={Annual Reviews}
}

@article{miasojedow2013adaptive,
  title={An adaptive parallel tempering algorithm},
  author={Miasojedow, B{\l}a{\.z}ej and Moulines, Eric and Vihola, Matti},
  journal={Journal of Computational and Graphical Statistics},
  volume={22},
  number={3},
  pages={649--664},
  year={2013},
  publisher={Taylor \& Francis}
}

@inproceedings{deng2023non,
  title={Non-reversible parallel tempering for deep posterior approximation},
  author={Deng, Wei and Zhang, Qian and Feng, Qi and Liang, Faming and Lin, Guang},
  booktitle={Proceedings of the AAAI Conference on Artificial Intelligence},
  volume={37},
  number={6},
  pages={7332--7339},
  year={2023}
}

@article{laio2002escaping,
  title={Escaping free-energy minima},
  author={Laio, Alessandro and Parrinello, Michele},
  journal={Proceedings of the National Academy of Sciences},
  volume={99},
  number={20},
  pages={12562--12566},
  year={2002},
  publisher={National Acad Sciences}
}

@article{dellago2002transition,
  title={Transition path sampling},
  author={Dellago, Christoph and Bolhuis, Peter G and Geissler, Phillip L},
  journal={Advances in Chemical Physics},
  volume={123},
  pages={1--84},
  year={2002},
  publisher={Wiley Online Library}
}

@article{dinh2016density,
  title={Density estimation using real nvp},
  author={Dinh, Laurent and Sohl-Dickstein, Jascha and Bengio, Samy},
  journal={arXiv preprint arXiv:1605.08803},
  year={2016}
}

@article{durkan2019neural,
  title={Neural spline flows},
  author={Durkan, Conor and Bekasov, Artur and Murray, Iain and Papamakarios, George},
  journal={Advances in neural information processing systems},
  volume={32},
  year={2019}
}

@article{coretti2024boltzmann,
  title={Boltzmann generators and the new frontier of computational sampling in many-body systems},
  author={Coretti, Alessandro and Falkner, Sebastian and Weinreich, Jan and Dellago, Christoph and von Lilienfeld, O Anatole},
  journal={arXiv preprint arXiv:2404.16566},
  year={2024}
}

@article{schebek2025scalable,
  title={Scalable Boltzmann Generators for equilibrium sampling of large-scale materials},
  author={Schebek, Maximilian and No{\'e}, Frank and Rogal, Jutta},
  journal={arXiv preprint arXiv:2509.25486},
  year={2025}
}

@article{li2023fast,
  title={Fast replica exchange stochastic gradient Langevin dynamics},
  author={Li, Guang and Lin, Guang and Zhang, Zecheng and Zhou, Qing},
  journal={arXiv preprint arXiv:2301.01898},
  year={2023}
}

@article{lin2023b,
  title={{B-DeepONet}: An enhanced {Bayesian} {DeepONet} for solving noisy parametric PDEs using accelerated replica exchange {SGLD}},
  author={Lin, Guang and Moya, Christian and Zhang, Zecheng},
  journal={Journal of Computational Physics},
  volume={473},
  pages={111713},
  year={2023},
  publisher={Elsevier}
}

@article{deng2020accelerating,
  title={Accelerating convergence of replica exchange stochastic gradient {MCMC} via variance reduction},
  author={Deng, Wei and Feng, Qi and Karagiannis, Georgios and Lin, Guang and Liang, Faming},
  journal={arXiv preprint arXiv:2010.01084},
  year={2020}
}

@article{herman1982path,
  title={On path integral {Monte Carlo} simulations},
  author={Herman, Michael F and Bruskin, Eric J and Berne, Bruce J},
  journal={The Journal of Chemical Physics},
  volume={76},
  number={10},
  pages={5150--5155},
  year={1982},
  publisher={American Institute of Physics}
}

@article{schoof2011configuration,
  title={Configuration path integral {Monte Carlo}},
  author={Schoof, Tim and Bonitz, Michael and Filinov, Alexey and Hochstuhl, David and Dufty, James W},
  journal={Contributions to Plasma Physics},
  volume={51},
  number={8},
  pages={687--697},
  year={2011},
  publisher={Wiley Online Library}
}

@article{marx1996ab,
  title={Ab initio path integral molecular dynamics: Basic ideas},
  author={Marx, Dominik and Parrinello, Michele},
  journal={The Journal of Chemical Physics},
  volume={104},
  number={11},
  pages={4077--4082},
  year={1996},
  publisher={American Institute of Physics}
}

@article{ceriotti2010efficient,
  title={Efficient stochastic thermostatting of path integral molecular dynamics},
  author={Ceriotti, Michele and Parrinello, Michele and Markland, Thomas E and Manolopoulos, David E},
  journal={The Journal of Chemical Physics},
  volume={133},
  number={12},
  year={2010},
  publisher={American Institute of Physics}
}

@book{sugiyama2012density,
  title={Density ratio estimation in machine learning},
  author={Sugiyama, Masashi and Suzuki, Taiji and Kanamori, Takafumi},
  year={2012},
  publisher={Cambridge University Press}
}

@incollection{tsybakov2008nonparametric,
  title={Nonparametric estimators},
  author={Tsybakov, Alexandre B},
  booktitle={Introduction to Nonparametric Estimation},
  pages={1--76},
  year={2008},
  publisher={Springer}
}

@article{betancourt2017conceptual,
  title={A conceptual introduction to Hamiltonian Monte Carlo},
  author={Betancourt, Michael},
  journal={arXiv preprint arXiv:1701.02434},
  year={2017}
}

@inproceedings{chen2014stochastic,
  title={Stochastic gradient hamiltonian monte carlo},
  author={Chen, Tianqi and Fox, Emily and Guestrin, Carlos},
  booktitle={International conference on machine learning},
  pages={1683--1691},
  year={2014},
  organization={PMLR}
}

@article{chib1995understanding,
  title={Understanding the metropolis-hastings algorithm},
  author={Chib, Siddhartha and Greenberg, Edward},
  journal={The american statistician},
  volume={49},
  number={4},
  pages={327--335},
  year={1995},
  publisher={Taylor \& Francis}
}

@incollection{robert2009metropolis,
  title={Metropolis--hastings algorithms},
  author={Robert, Christian P and Casella, George},
  booktitle={Introducing Monte Carlo Methods with R},
  pages={167--197},
  year={2009},
  publisher={Springer}
}

@book{meyn2012markov,
  title={Markov chains and stochastic stability},
  author={Meyn, Sean P and Tweedie, Richard L},
  year={2012},
  publisher={Springer Science \& Business Media}
}

@article{noe2019boltzmann,
  title={Boltzmann generators: Sampling equilibrium states of many-body systems with deep learning},
  author={No{\'e}, Frank and Olsson, Simon and K{\"o}hler, Jonas and Wu, Hao},
  journal={Science},
  volume={365},
  number={6457},
  pages={eaaw1147},
  year={2019},
  publisher={American Association for the Advancement of Science},
  OPTnote={CATEGORY: Foundational Flow-based Sampling. LOSS FUNCTION: Energy KL (Reverse KL) + ML (Forward KL on transition states). ANNOTATION: The foundational paper on Boltzmann Generators. They primarily train using the Energy KL (Reverse KL) divergence which directly utilizes the target potential energy. While this yields high-fidelity samples in low-energy regions, it is intrinsically mode-seeking and suffers from severe mode collapse. They sometimes mix in a Forward KL loss on specific transition state data, but rely on reaction coordinates rather than a unified symmetrized distillation like Jeffreys Flow.}
}

@article{qiu2024efficient,
  title={Efficient multimodal sampling via tempered distribution flow},
  author={Qiu, Yixuan and Wang, Xiao},
  journal={Journal of the American Statistical Association},
  volume={119},
  number={546},
  pages={1446--1460},
  year={2024},
  publisher={Taylor \& Francis},
  OPTnote={CATEGORY: Mode Collapse Mitigation. LOSS FUNCTION: Squared L2 Wasserstein distance. ANNOTATION: TemperFlow explicitly replaces the standard KL divergence objectives with the squared L2 Wasserstein distance. The theoretical motivation is that KL divergences (especially reverse KL) suffer from vanishing gradients when isolated modes have no overlap with the current model density. In contrast, Jeffreys Flow solves the vanishing gradient / mode collapse issue by symmetrically combining Forward KL (which penalizes missing mass) and Reverse KL, relying on empirical Parallel Tempering samples to anchor the mass rather than computing optimal transport metrics.}
}

@article{gabrie2022adaptive,
  title={Adaptive Monte Carlo augmented with normalizing flows},
  author={Gabri{\'e}, Marylou and Rotskoff, Grant M and Vanden-Eijnden, Eric},
  journal={Proceedings of the National Academy of Sciences},
  volume={119},
  number={10},
  pages={e2109420119},
  year={2022},
  publisher={National Academy of Sciences},
  OPTnote={CATEGORY: Mode Collapse Mitigation / MCMC Augmentation. LOSS FUNCTION: Data KL (Forward KL). ANNOTATION: This paper trains flows via Maximum Likelihood Estimation (Forward KL) on samples collected on-the-fly from an adaptive MCMC chain. The use of Forward KL ensures mass-covering behavior, but the resulting flow's accuracy is strictly bounded by the empirical MCMC data. Jeffreys Flow improves upon this by injecting the exact physical potential via the Energy KL (Reverse KL) term, which acts as a physics-informed regularizer to surpass the empirical data's quality.}
}

@article{falkner2023conditioning,
  title={Conditioning Boltzmann generators for rare event sampling},
  author={Falkner, Sebastian and Coretti, Alessandro and Romano, Salvatore and Geissler, Phillip L and Dellago, Christoph},
  journal={Machine Learning: Science and Technology},
  volume={4},
  number={3},
  pages={035050},
  year={2023},
  publisher={IOP Publishing},
  OPTnote={CATEGORY: Rare Event Sampling. LOSS FUNCTION: Conditional Reverse KL. ANNOTATION: Conditional Boltzmann Generators (CBG) condition the normalizing flow on predefined reaction coordinates to steer sampling across energy barriers. While effective for mode traversal, it relies heavily on domain-specific knowledge. Jeffreys Flow provides a fully unconditional alternative, leveraging ergodically mixed PT samples and the mass-covering properties of the Data KL term to traverse barriers without manual collective variables.}
}

@article{dibak2022temperature,
  title={Temperature steerable flows and Boltzmann generators},
  author={Dibak, Manuel and Klein, Leon and Kr{\"a}mer, Andreas and No{\'e}, Frank},
  journal={Physical Review Research},
  volume={4},
  number={4},
  pages={L042005},
  year={2022},
  publisher={APS},
  OPTnote={CATEGORY: Temperature Annealing Flows. LOSS FUNCTION: Reverse KL (parameterized by temperature). ANNOTATION: Introduces a single global normalizing flow parameterized by a continuous temperature variable, trained using the Reverse KL at various sampled temperatures. Jeffreys Flow, conversely, trains a discrete sequence of localized flows using the Dual-KL objective, which mathematically guarantees the containment of the instability set at each explicit PT step.}
}

@article{midgley2022flow,
  title={Flow annealed importance sampling bootstrap},
  author={Midgley, Laurence Illing and Stimper, Vincent and Simm, Gregor NC and Sch{\"o}lkopf, Bernhard and Hern{\'a}ndez-Lobato, Jos{\'e} Miguel},
  journal={arXiv preprint arXiv:2208.01893},
  year={2022},
  OPTnote={CATEGORY: Sequential Flow Training. LOSS FUNCTION: Reverse KL (Annealed). ANNOTATION: Flow Annealed Importance Sampling Bootstrap (FAIS) sequentially trains normalizing flows to sample multimodal distributions by bridging a prior to the target. However, pure sequential training without a target-aware dataset can still suffer from mode dropping if an intermediate step loses a mode. Jeffreys Flow prevents this by distilling from pre-mixed Parallel Tempering samples using the mass-covering Forward KL alongside the Reverse KL.}
}

@article{wang2025mitigating,
  title={Mitigating mode collapse in normalizing flows by annealing with an adaptive schedule: Application to parameter estimation},
  author={Wang, Yihang and Chi, Chris and Dinner, Aaron R},
  journal={arXiv preprint arXiv:2505.03652},
  year={2025},
  OPTnote={CATEGORY: Mode Collapse Mitigation. LOSS FUNCTION: Reverse KL (with adaptive ESS scheduling). ANNOTATION: This approach mitigates mode collapse by dynamically adjusting the annealing schedule based on a minimum Effective Sample Size (ESS) between intermediate distributions. While theoretically sound, the ladder can become prohibitively fine-grained. Jeffreys Flow avoids infinite ladder expansion by relying on the structural robustness of the Jeffreys Divergence to force global coverage even with larger step sizes.}
}

@article{wu2020stochastic,
  title={Stochastic normalizing flows},
  author={Wu, Hao and K{\"o}hler, Jonas and No{\'e}, Frank},
  journal={Advances in neural information processing systems},
  volume={33},
  pages={5933--5944},
  year={2020},
  OPTnote={CATEGORY: MCMC + Flow Integration. LOSS FUNCTION: Reverse KL (with stochastic transition paths). ANNOTATION: SNFs interleave deterministic normalizing flow layers with stochastic MCMC sampling steps to improve expressiveness and mode-hopping. Jeffreys Flow also incorporates MCMC (via MALA rejuvenation) on intermediate validation sets, but its core mapping relies entirely on the deterministic distillation of the Dual-KL objective.}
}

@article{matthews2022continual,
  title={Continual repeated annealed flow transport monte carlo},
  author={Matthews, Alexander GDG and Arbel, Michael and Rezende, Danilo J and Doucet, Arnaud},
  journal={arXiv preprint arXiv:2201.13117},
  year={2022},
  OPTnote={CATEGORY: Sequential Flow Training. LOSS FUNCTION: Reverse KL (Annealed SMC). ANNOTATION: CRAFT applies repeated annealed flow transport using MCMC and normalizing flows. Like FAIS, it relies heavily on sequential transitions where errors can cascade. Jeffreys Flow uniquely pairs the distillation sequence with Parallel Tempering swap mechanics to ensure global ergodicity is injected directly into the training data before the Forward KL anchors it.}
}

@article{syed2022non,
  title={Non-reversible parallel tempering: a scalable highly parallel MCMC scheme},
  author={Syed, Saifuddin and Bouchard-C{\^o}t{\'e}, Alexandre and Deligiannidis, George and Doucet, Arnaud},
  journal={Journal of the Royal Statistical Society Series B: Statistical Methodology},
  volume={84},
  number={2},
  pages={321--350},
  year={2022},
  publisher={Oxford University Press},
  OPTnote={CATEGORY: Baseline MCMC. LOSS FUNCTION: N/A (MCMC detailed balance). ANNOTATION: Represents the gold-standard Parallel Tempering (PT) methods that Jeffreys Flow seeks to distill. Rather than trying to design a flow-based sampler that outperforms PT in raw exploration, Jeffreys Flow extracts the high-quality, ergodically mixed samples produced by PT.}
}

@article{abbott2023normalizing,
  title={Normalizing flows for lattice gauge theory in arbitrary space-time dimension},
  author={Abbott, Ryan and Albergo, Michael S and Botev, Aleksandar and Boyda, Denis and Cranmer, Kyle and Hackett, Daniel C and Kanwar, Gurtej and Matthews, Alexander GDG and Racani{\`e}re, S{\'e}bastien and Razavi, Ali and others},
  journal={arXiv preprint arXiv:2305.02402},
  year={2023},
  OPTnote={CATEGORY: High-Dimensional Physics (PIMC/Lattice). LOSS FUNCTION: Reverse KL. ANNOTATION: Applies normalizing flows to high-dimensional lattice systems, fundamentally related to the infinite-dimensional Path Integral Monte Carlo (PIMC) target. Training full-dimensional flows on lattices faces exponential scaling bottlenecks. Jeffreys Flow bypasses this by exclusively applying the transport map to a truncated set of low-frequency Fourier normal modes, mathematically reweighting the residual high-frequency dimensions.}
}

@article{bayarri2008generalization,
  title={Generalization of Jeffreys divergence-based priors for Bayesian hypothesis testing},
  author={Bayarri, MJ and Garc{\'\i}a-Donato, G},
  journal={Journal of the Royal Statistical Society Series B: Statistical Methodology},
  volume={70},
  number={5},
  pages={981--1003},
  year={2008},
  publisher={Oxford University Press},
  OPTnote={CATEGORY: Symmetrized KL Applications. LOSS FUNCTION: Jeffreys Divergence. ANNOTATION: A purely statistical paper discussing non-informative priors based on the Jeffreys divergence. It highlights the fundamental theoretical stability of symmetrizing the KL divergence to enforce mutual absolute continuity, a property leveraged in Theorem 2 of the Jeffreys Flow paper to guarantee bounds on log-likelihood ratio deviations.}
}

@article{abernethy2024flow,
  title={Flow to Rare Events: An Application of Normalizing Flow in Temporal Importance Sampling for Automated Vehicle Validation},
  author={Abernethy, J and others},
  journal={arXiv preprint},
  year={2024},
  OPTnote={CATEGORY: Rare Event Sampling. LOSS FUNCTION: Forward KL (Cross-Entropy). ANNOTATION: Applies normalizing flows specifically to estimate rare event probabilities in temporal dynamics. They primarily use a Cross-Entropy (Forward KL) objective trained on failure data. In contrast, Jeffreys Flow handles rare event simulation across energy barriers by explicitly anchoring the generative model to the exact physical Hamiltonian via the Reverse KL, creating a continuous bridge rather than just fitting rare failure samples.}
}

@article{wirnsberger2020targeted,
  title={Targeted free energy estimation via normalizing flows},
  author={Wirnsberger, Peter and Ballard, Andrew J and Papamakarios, George and Abercrombie, Stuart and Racani{\`e}re, S{\'e}bastien and Pritzel, Alexander and Jimenez Rezende, Danilo and Blundell, Charles},
  journal={The Journal of Chemical Physics},
  volume={153},
  number={14},
  pages={144112},
  year={2020},
  publisher={AIP Publishing LLC},
  OPTnote={CATEGORY: Rare Event Sampling / Thermodynamics. LOSS FUNCTION: Reverse KL (Targeted). ANNOTATION: Employs normalizing flows to compute free energy differences between thermodynamic states, essentially addressing mode transitions. They rely on Reverse KL training and targeted reweighting. Jeffreys Flow extends this thermodynamic bridging by explicitly tracking the Parallel Tempering ladder and using the Jeffreys Divergence to ensure bidirectional fidelity between adjacent free energy basins.}
}

@inproceedings{kohler2024transferable,
  title={Transferable Boltzmann Generators},
  author={K{\"o}hler, Jonas and Ingraham, John and No{\'e}, Frank},
  booktitle={Advances in Neural Information Processing Systems},
  volume={37},
  pages={31980--31993},
  year={2024}
}

@article{midgley2023learning,
  title={Learning to sample with Flow Annealed Importance Sampling Bootstrap},
  author={Midgley, Laurence Illing and Stimper, Vincent and Simm, Gregor NC and Sch{\"o}lkopf, Bernhard and Hern{\'a}ndez-Lobato, Jos{\'e} Miguel},
  journal={The Eleventh International Conference on Learning Representations},
  year={2023}
}

@article{earl2005parallel,
	title={Parallel tempering: Theory, applications, and new perspectives},
	author={Earl, David J and Deem, Michael W},
	journal={Physical Chemistry Chemical Physics},
	volume={7},
	number={23},
	pages={3910--3916},
	year={2005},
	publisher={Royal Society of Chemistry}
}

@article{luhman2021knowledge,
	title={Knowledge distillation in iterative generative models for improved sampling speed},
	author={Luhman, Eric and Luhman, Troy},
	journal={arXiv preprint arXiv:2101.02388},
	year={2021}
}

@article{zhou2024simple,
  title={Simple and fast distillation of diffusion models},
  author={Zhou, Zhenyu and Chen, Defang and Wang, Can and Chen, Chun and Lyu, Siwei},
  journal={Advances in Neural Information Processing Systems},
  volume={37},
  pages={40831--40860},
  year={2024}
}

@article{xie2024distillation,
	title={Em distillation for one-step diffusion models},
	author={Xie, Sirui and Xiao, Zhisheng and Kingma, Diederik and Hou, Tingbo and Wu, Ying Nian and Murphy, Kevin P and Salimans, Tim and Poole, Ben and Gao, Ruiqi},
	journal={Advances in Neural Information Processing Systems},
	volume={37},
	pages={45073--45104},
	year={2024}
}

@inproceedings{fu2025moflow,
	title={Moflow: One-step flow matching for human trajectory forecasting via implicit maximum likelihood estimation based distillation},
	author={Fu, Yuxiang and Yan, Qi and Wang, Lele and Li, Ke and Liao, Renjie},
	booktitle={Proceedings of the Computer Vision and Pattern Recognition Conference},
	pages={17282--17293},
	year={2025}
}

@article{song2021maximum,
	title={Maximum likelihood training of score-based diffusion models},
	author={Song, Yang and Durkan, Conor and Murray, Iain and Ermon, Stefano},
	journal={Advances in neural information processing systems},
	volume={34},
	pages={1415--1428},
	year={2021}
}

@article{batzolis2021conditional,
	title={Conditional image generation with score-based diffusion models},
	author={Batzolis, Georgios and Stanczuk, Jan and Sch{\"o}nlieb, Carola-Bibiane and Etmann, Christian},
	journal={arXiv preprint arXiv:2111.13606},
	year={2021}
}

@article{lipman2022flow,
  title={Flow matching for generative modeling},
  author={Lipman, Yaron and Chen, Ricky TQ and Ben-Hamu, Heli and Nickel, Maximilian and Le, Matt},
  journal={arXiv preprint arXiv:2210.02747},
  year={2022}
}

@article{chen2023flow,
  title={Flow matching on general geometries},
  author={Chen, Ricky TQ and Lipman, Yaron},
  journal={arXiv preprint arXiv:2302.03660},
  year={2023}
}

@article{han2020active,
  title={Active contour model for inhomogenous image segmentation based on Jeffreys divergence},
  author={Han, Bin and Wu, Yuting},
  journal={Pattern Recognition},
  volume={107},
  pages={107520},
  year={2020},
  publisher={Elsevier}
}

@article{sharma2021clustering,
  title={Clustering uncertain data objects using Jeffreys-divergence and maximum bipartite matching based similarity measure},
  author={Sharma, Krishna Kumar and Seal, Ayan and Yazidi, Anis and Selamat, Ali and Krejcar, Ondrej},
  journal={IEEE Access},
  volume={9},
  pages={79505--79519},
  year={2021},
  publisher={IEEE}
}

@article{nishii2006image,
  title={Image classification based on Markov random field models with Jeffreys divergence},
  author={Nishii, Ryuei and Eguchi, Shinto},
  journal={Journal of Multivariate Analysis},
  volume={97},
  number={9},
  pages={1997--2008},
  year={2006},
  publisher={Elsevier}
}

@inproceedings{moreno2003kullback,
  title={A Kullback-Leibler divergence based kernel for SVM classification in multimedia applications},
  author={Moreno, Pedro and Ho, Purdy and Vasconcelos, Nuno},
  booktitle={Advances in Neural Information Processing Systems},
  volume={16},
  year={2003}
}

@inproceedings{yao2011symmetric,
  title={A symmetric KL divergence based spatiogram similarity measure},
  author={Yao, Zhice and Lai, Zhihui and Liu, Wanquan},
  booktitle={2011 18th IEEE International Conference on Image Processing},
  pages={193--196},
  year={2011},
  organization={IEEE}
}

\begin{figure*}
	\vspace{2.4cm}
	\centering
	\includegraphics[width=\textwidth]{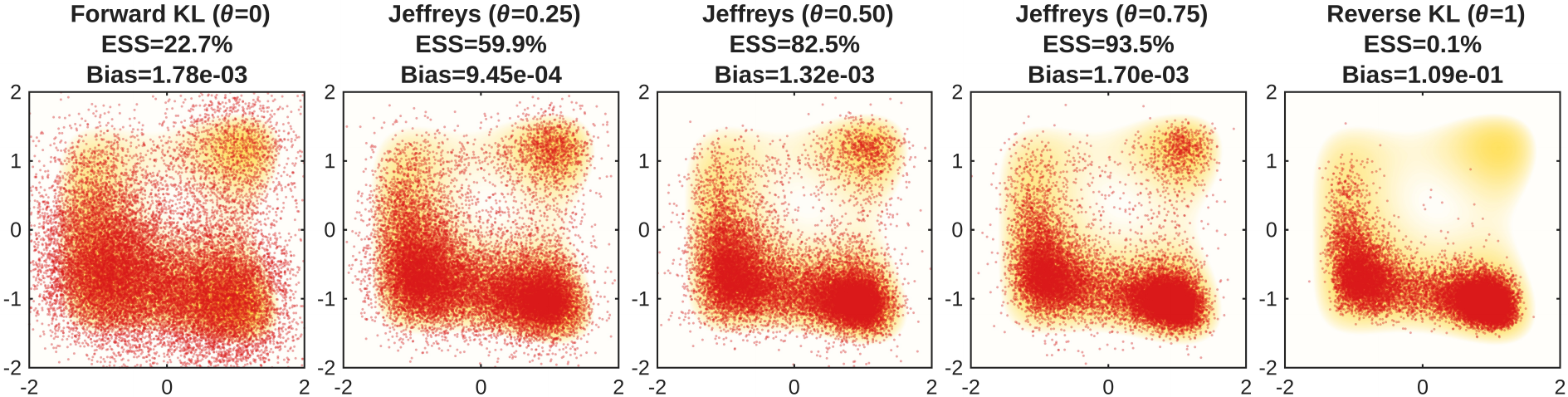}
	\includegraphics[width=\textwidth]{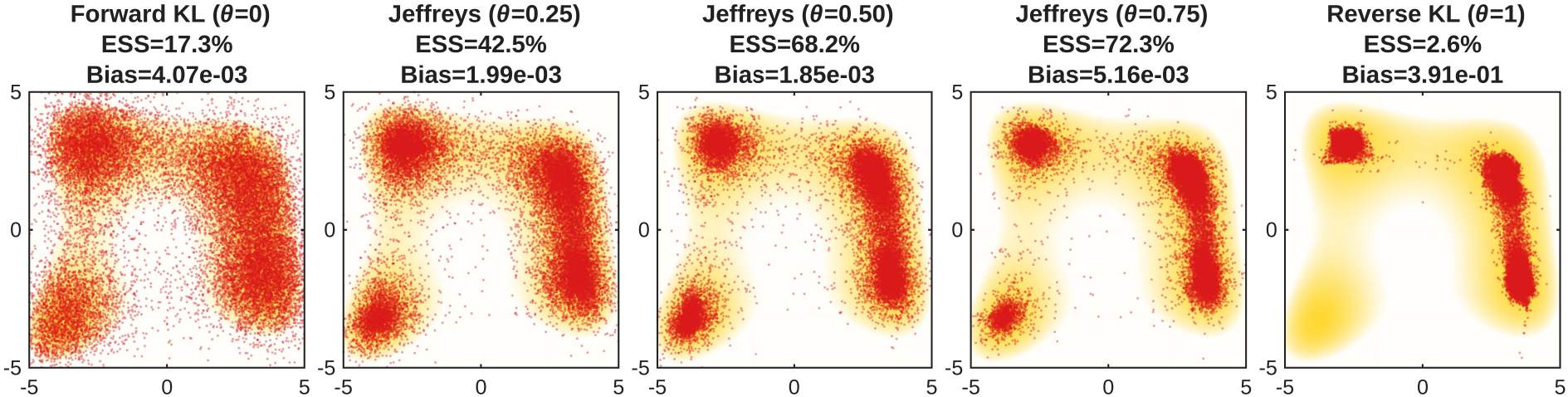}
	\includegraphics[width=\textwidth]{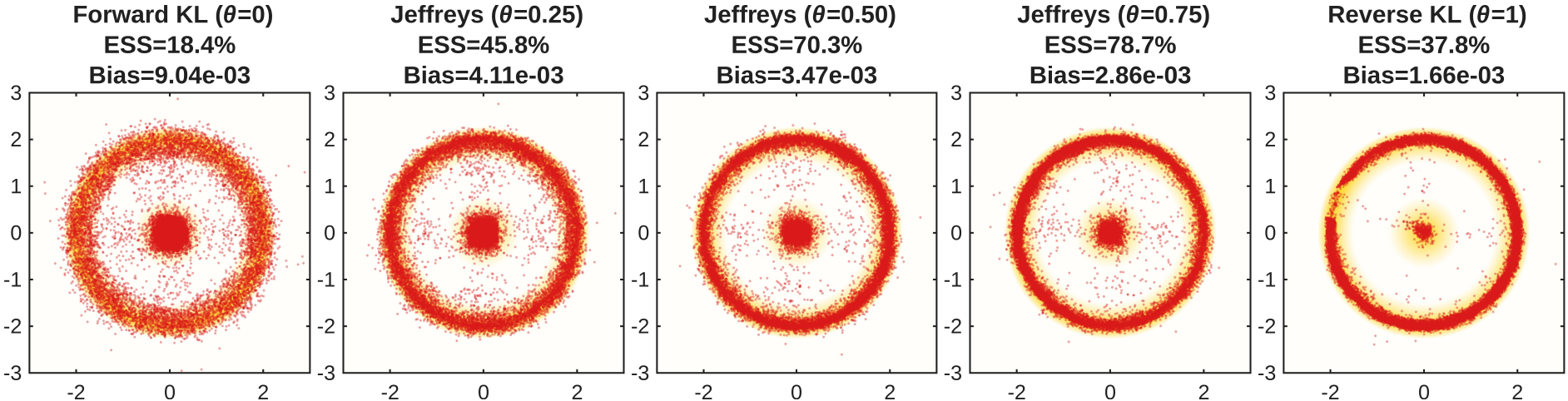}
	\includegraphics[width=\textwidth]{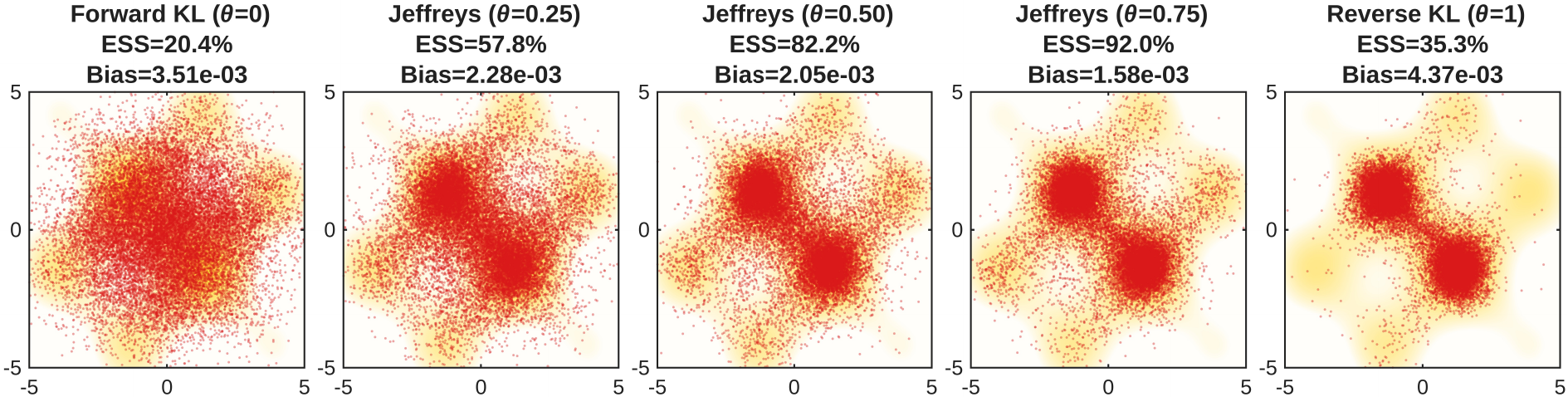}
	\caption{Density plots of the generated samples across four 2D landscapes. The Jeffreys Flow ($0 < \theta < 1$) robustly corrects the noisy artifacts of Forward KL ($\theta=0$) and strictly prevents the destructive mode collapse inherent to Reverse KL ($\theta=1$).}
	\label{fig: single_step}
\end{figure*}

\begin{figure*}
	\centering
	\includegraphics[width=\textwidth]{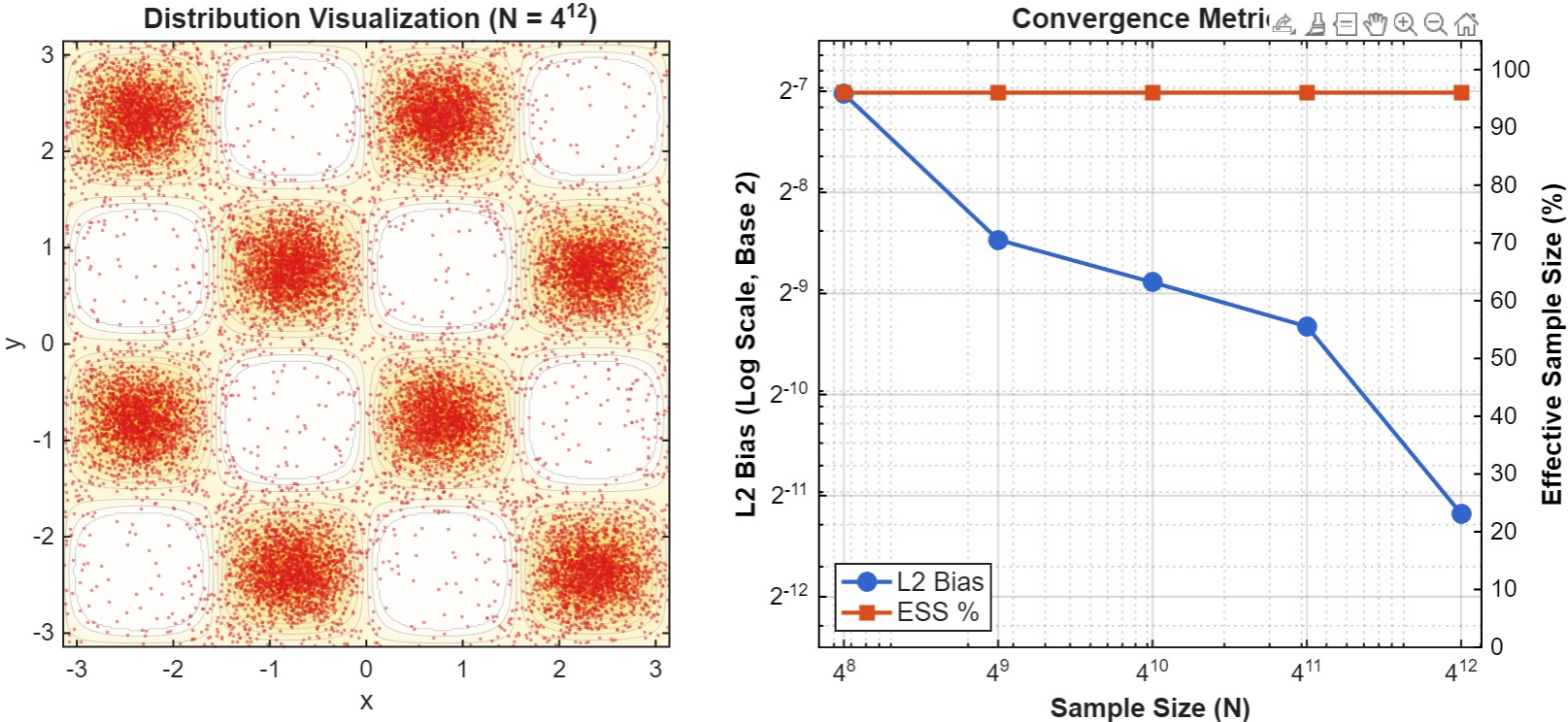}
	\caption{Left: The generated distribution from the Jeffreys Flow. The scatter samples accurately capture all modes of the periodic potential ($N=4^{12}$). Right: As sample size $N$ increases, the ESS remains consistently optimal ($\approx 100\%$) while the $L^2$ bias strictly follows the theoretical Monte Carlo convergence rate.}
	\label{fig: single_step_periodic}
\end{figure*}

\begin{figure*}
	\centering
	\includegraphics[width=\textwidth]{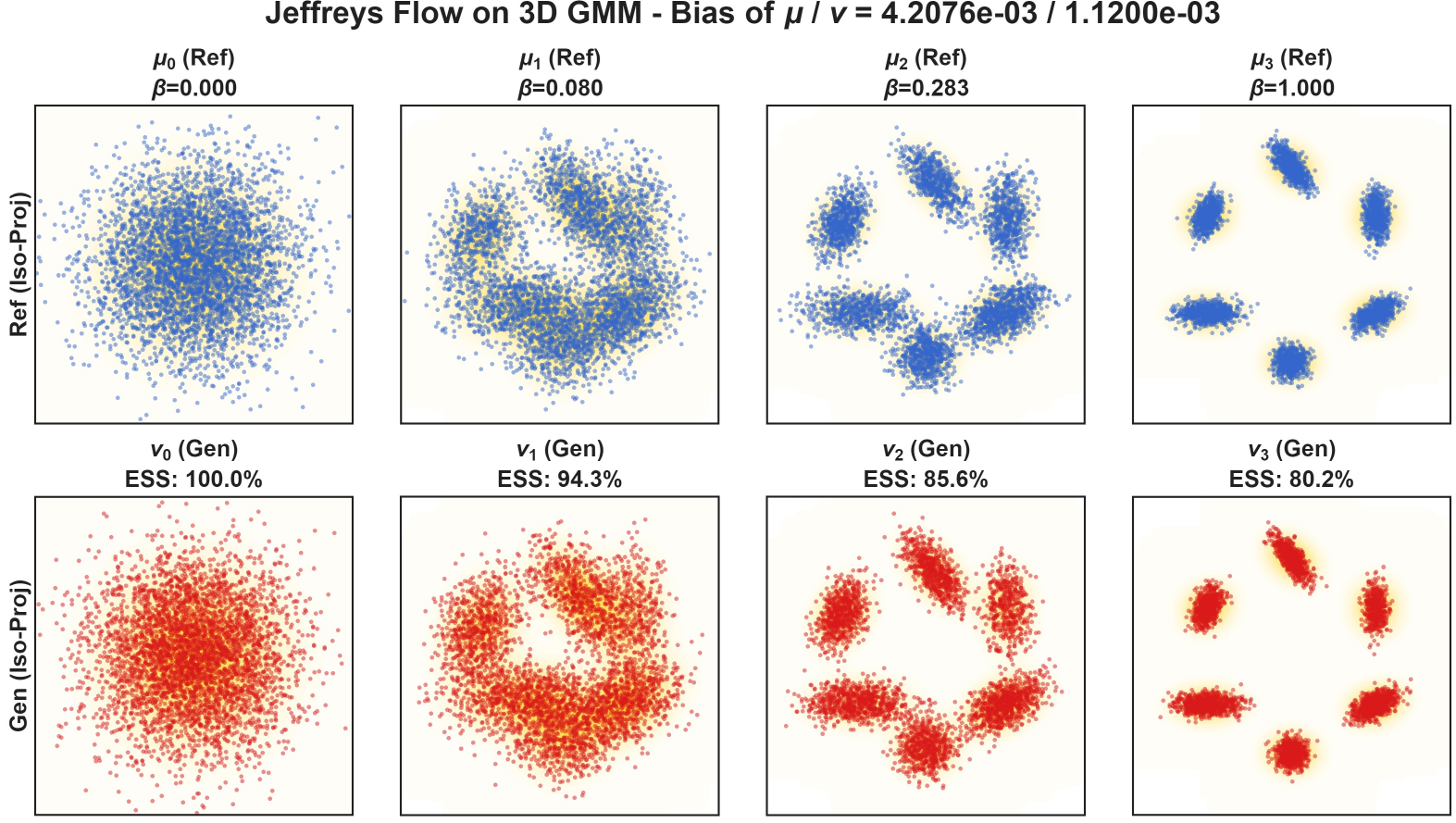}
	\caption{Sequential distillation on the 3D Gaussian Mixture Model. Top: Reference PT samples at various temperatures. Bottom: Generated samples from the distilled Jeffreys Flow. The resulting bias of the Jeffreys Flow is much smaller than PT.}
	\label{fig: gm_results}
\end{figure*}

\begin{figure*}
	\centering
	\includegraphics[width=0.9\textwidth]{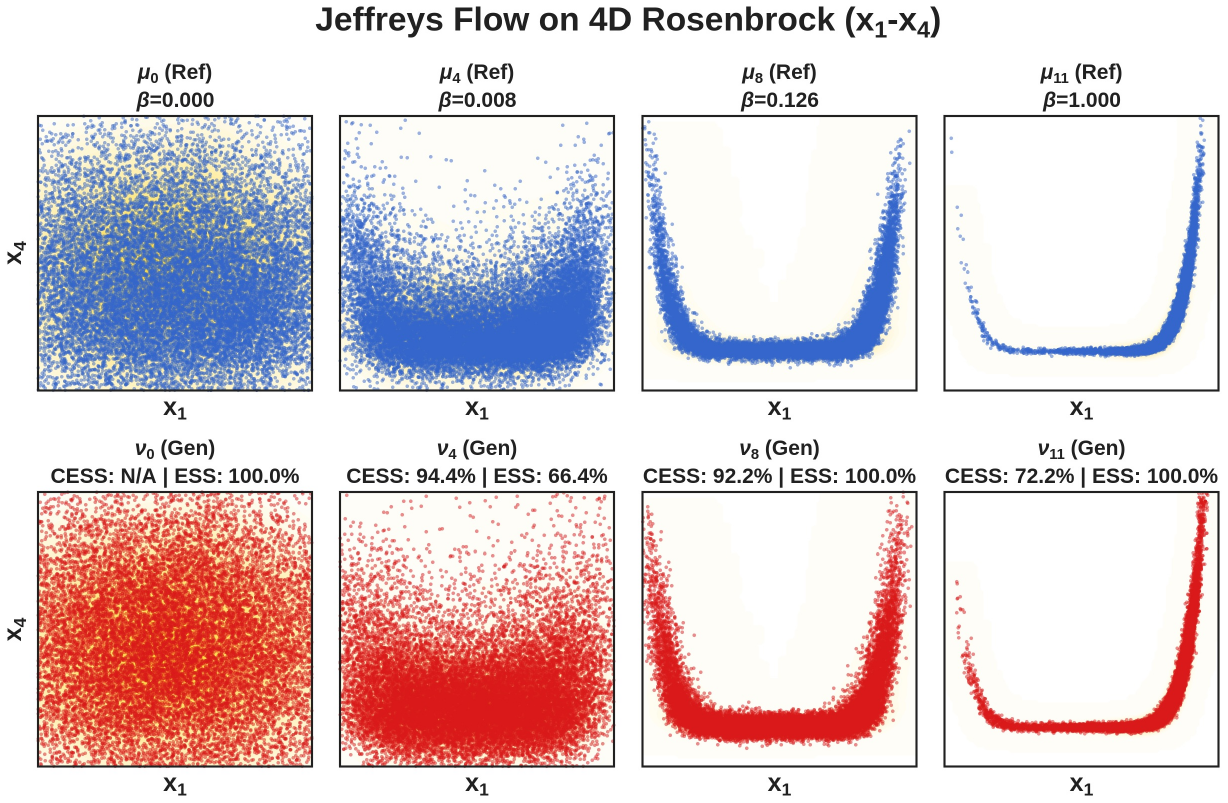}
	\caption{Sequential distillation on the 4D Rosenbrock. Top: Reference PT samples. Bottom: Generated samples from the Jeffreys Flow successfully learning the highly nonlinear correlations.}
	\label{fig: rb_results}
\end{figure*}

\begin{figure*}
	\centering
	\includegraphics[width=0.9\textwidth]{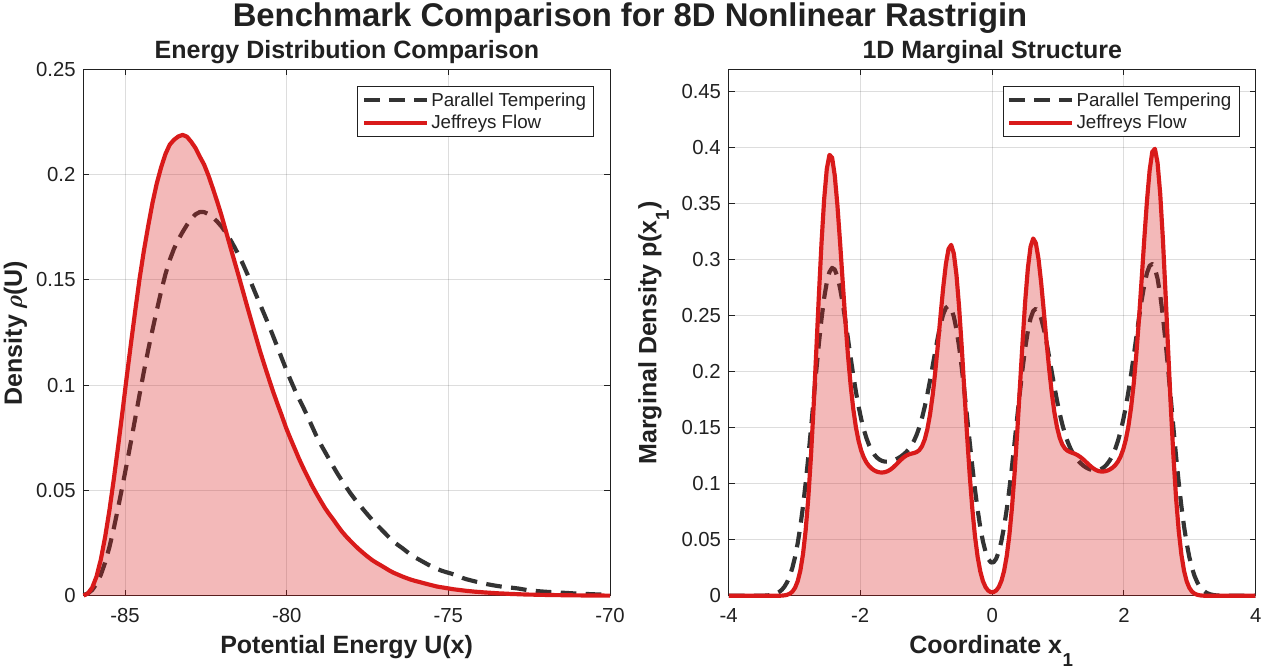}
	\caption{Comparison for 8D Nonlinear Rastrigin. Left: Potential energy distribution $P(U)$ comparison between PT and Jeffreys Flow. Right: 1D marginal distribution $p(x_1)$ comparison. Both methods show consistent mode coverage and energy profiles.}
	\label{fig: nr_results}
\end{figure*}

\begin{table*}[htbp]
    \centering
    \setlength{\tabcolsep}{4.5pt}
	\setstretch{1.2}
    \begin{tabular}{c |*{11}{c}}
       \hline
       Step & 1 & 2 & 3 & 4 & 5 & 6 & 7 & 8 & 9 & 10 & 11 \\
	   \hline
       $\beta$ & 0.020 & 0.032 & 0.050 & 0.100 & 0.150 & 0.200 & 0.250 & 0.300 & 0.350 & 0.400 & 0.450 \\
       CESS & 93.7\% & 95.9\% & 91.6\% & 77.6\% & 70.7\% & 78.1\% & 81.4\% & 87.7\% & 89.9\% & 92.3\% & 94.1\% \\
       ESS & 93.7\% & 88.5\% & 78.1\% & 57.1\% & \fbox{100\%} & 78.1\% & 57.0\% & \fbox{100\%} & 89.9\% & 80.1\% & 71.5\% \\
       \hline\hline
       Step & 12 & 13 & 14 & 15 & 16 & 17 & 18 & 19 & 20 & 21 & 22 \\
	   \hline
       $\beta$ & 0.500 & 0.550 & 0.600 & 0.650 & 0.700 & 0.750 & 0.800 & 0.850 & 0.900 & 0.950 & 1.000 \\
       CESS & 94.6\% & 95.4\% & 95.4\% & 95.9\% & 95.7\% & 96.1\% & 95.5\% & 95.9\% & 95.0\% & 95.3\% & 95.0\% \\
       ESS & 63.0\% & 54.4\% & \fbox{100\%} & 95.9\% & 88.8\% & 79.3\% & 66.7\% & 52.5\% & \fbox{100\%} & 95.3\% & 85.3\% \\
       \hline 
    \end{tabular}
	\caption{Performance metrics for the 8D Nonlinear Rastrigin potential. The table displays the inverse temperature $\beta$, the CESS, and the ESS at each stage. The ESS is reset to $100\%$ (marked as boxed) whenever adaptive resampling is performed.}
	\label{tab: nr_metrics}
\end{table*}

\begin{figure*}
	\centering
	\includegraphics[width=0.6\textwidth]{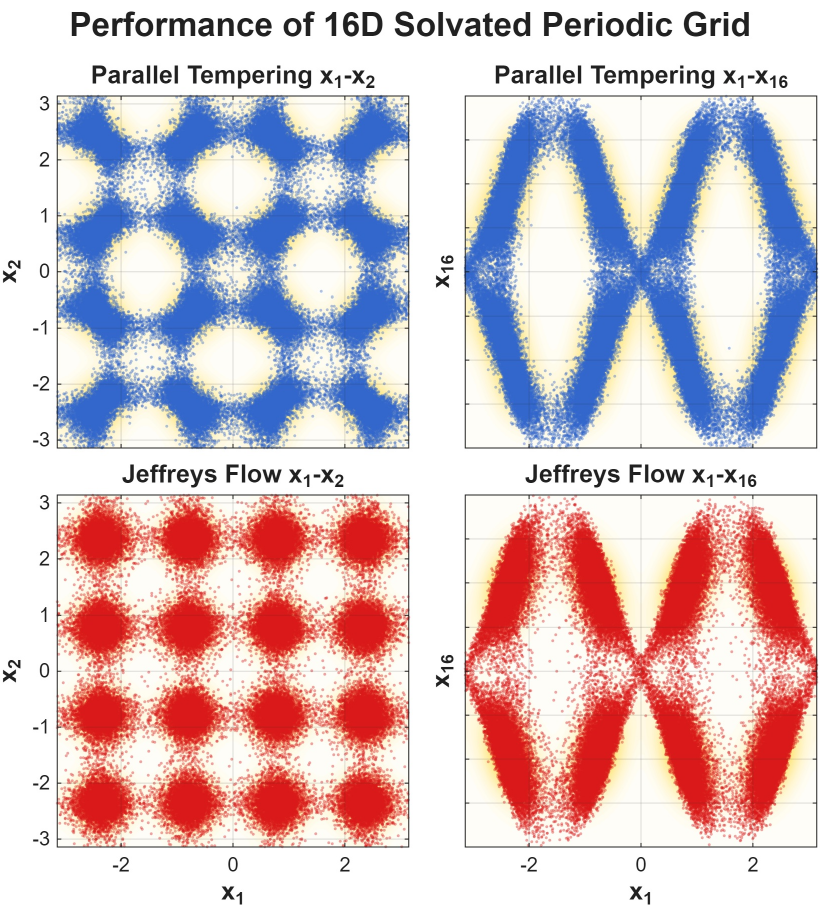}
	\caption{Decorrelation in the 16D Solvated Periodic Grid. Top: PT reference samples exhibit spurious diagonal correlations. Bottom: Jeffreys Flow successfully uncouples the bath to recover the exact independent checkerboard structure.}
	\label{fig: sl_joint}
\end{figure*}

\begin{figure*}
	\centering
	\includegraphics[width=0.4\textwidth]{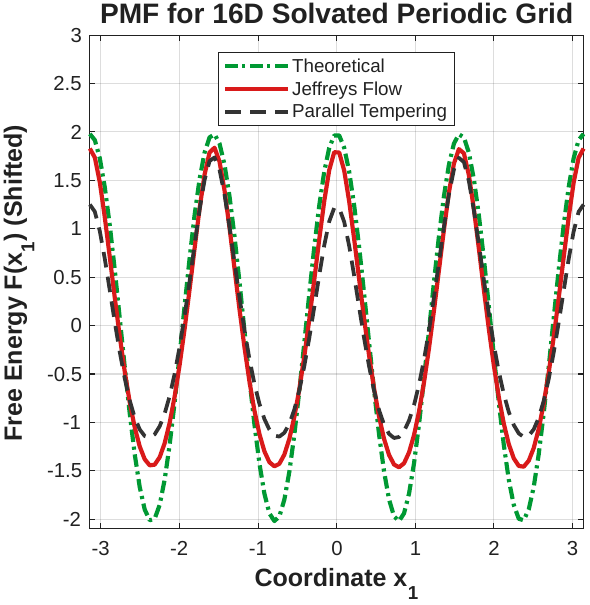}
	\caption{Potential of Mean Force (PMF) along $x_1$. Jeffreys Flow perfectly reconstructs the analytical free energy barriers that PT completely fails to cross.}
	\label{fig: sl_pmf}
\end{figure*}

\begin{figure*}
	\centering
	\includegraphics[width=0.6\textwidth]{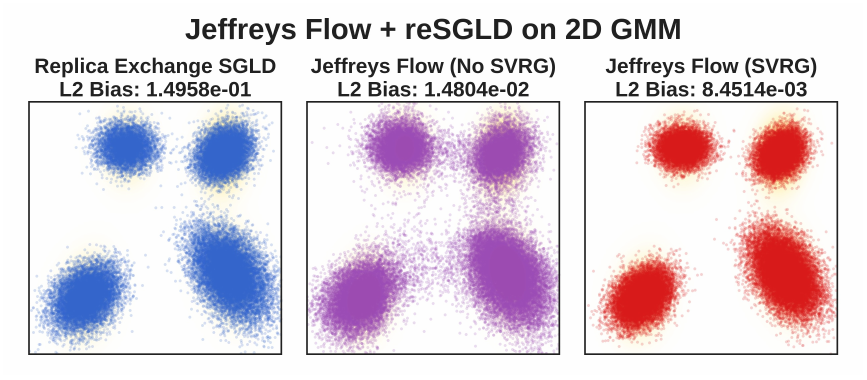}
	\caption{Decorrelation in 2D GMM. Visual comparison between reSGLD and Jeffreys Flow. Jeffreys (SVRG) fully resolves the probability density distribution. Jeffreys (No SVRG) generates samples with excessive spread and undesirable artificial tails.}
	\label{fig: resgld_results}
\end{figure*}

\begin{table*}[htbp]
    \centering
    \setlength{\tabcolsep}{4.5pt}
	\setstretch{1.2}
    \begin{tabular}{c |*{6}{c}}
        \hline
        Step & 0 & 1 & 2 & 3 & 4 & 5 \\
        $\beta$ & 0.0 & 0.05 & 0.2 & 0.4 & 0.7 & 1.0 \\
        \hline\hline
        reSGLD & 5.97e-02 & 5.90e-02 & 6.11e-02 & 7.49e-02 & 1.12e-01 & 1.50e-01 \\
        Jeffreys (No SVRG) & 2.81e-03 & 2.57e-02 & 1.91e-02 & 1.96e-02 & 1.32e-02 & 1.48e-02 \\
        Jeffreys (SVRG) & 4.94e-03 & 1.65e-03 & 1.56e-02 & 6.44e-03 & 4.60e-03 & 8.45e-03 \\
        \hline
    \end{tabular}
    \caption{Bias Evaluation across the intermediate targets in the 2D GMM experiment. The bias is drastically reduced by training with Jeffreys Flow, filtering out aggressive discretization errors inherent in reSGLD.}
    \label{tab: resgld_bias}
\end{table*}

\begin{table*}[htbp]
    \centering
	\setlength{\tabcolsep}{6pt}
	\setstretch{1.2}
    \begin{tabular}{c |*{6}{c}}
        \hline
        Step & 0 & 1 & 2 & 3 & 4 & 5 \\
        $\beta$ & 0.0 & 0.05 & 0.2 & 0.4 & 0.7 & 1.0 \\
        \hline\hline
        Jeffreys (No SVRG) & -- & 99.17\% & 97.52\% & 94.42\% & 87.17\% & 76.80\% \\
        Jeffreys (SVRG) & -- & 99.25\% & 96.77\% & 95.34\% & 91.77\% & 86.02\% \\
        \hline
    \end{tabular}
    \caption{ESS across the intermediate targets in the 2D GMM experiment. High quantitative ESS values demonstrate the ability of Jeffreys Flow as a robust feed-forward Boltzmann generator for direct, independent sample generation.}
    \label{tab: resgld_ess}
\end{table*}

\begin{table*}[htbp]
	\centering
	\setlength{\tabcolsep}{6pt}
	\setstretch{1.2}
	\begin{tabular}{c |*{8}{c}}
		\hline
		Step & 1 & 2 & 3 & 4 & 5 & 6 & 7 & 8 \\
		\hline\hline
		Exact & 83.93\% & 91.49\% & 92.48\% & 88.69\% & 77.30\% & 85.23\% & 89.56\% & 89.24\% \\
		Full Potential & 83.42\% & 88.28\% & 93.09\% & 88.83\% & 79.90\% & 86.18\% & 89.25\% & 91.15\% \\
		Stochastic & 61.71\% & 64.27\% & 84.69\% & 90.00\% & 89.08\% & 93.23\% & 79.66\% & 94.14\% \\
		\hline
	\end{tabular}
	\caption{CESS across the intermediate targets in the 2D Screened Poisson experiment. All evaluated configurations sustain extremely high CESS scores, confirming the robustness and massive acceleration achieved by the trained Jeffreys Flow.}
	\label{tab: poisson_ess}
\end{table*}

\begin{figure*}
	\centering
	\includegraphics[width=\textwidth]{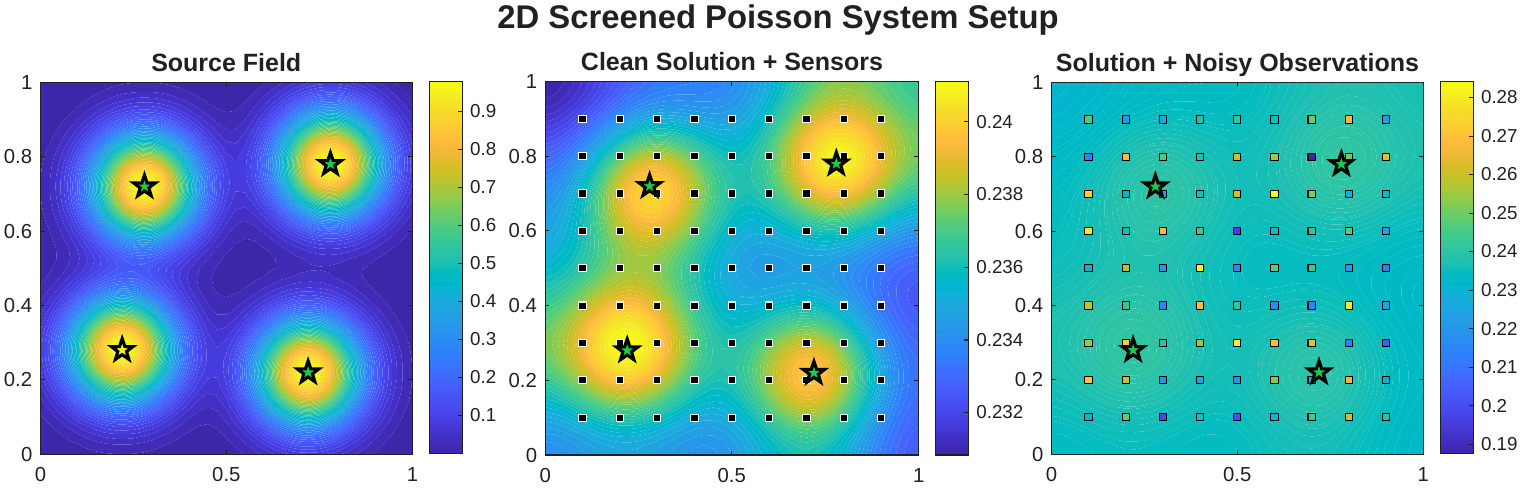}
	\caption{2D Screened Poisson Equipment Setup. Left: The true multi-source distribution field generated by the ground truth locations. Right: The resulting physical response $u(x)$ modeled by solving the discretized Poisson PDE. The $81$ fixed local sensors (square markers) collect the noisy measurements used to infer the unknown source locations (star markers).}
	\label{fig: poisson_setup}
\end{figure*}

\begin{figure*}
	\centering
	\includegraphics[width=0.75\textwidth]{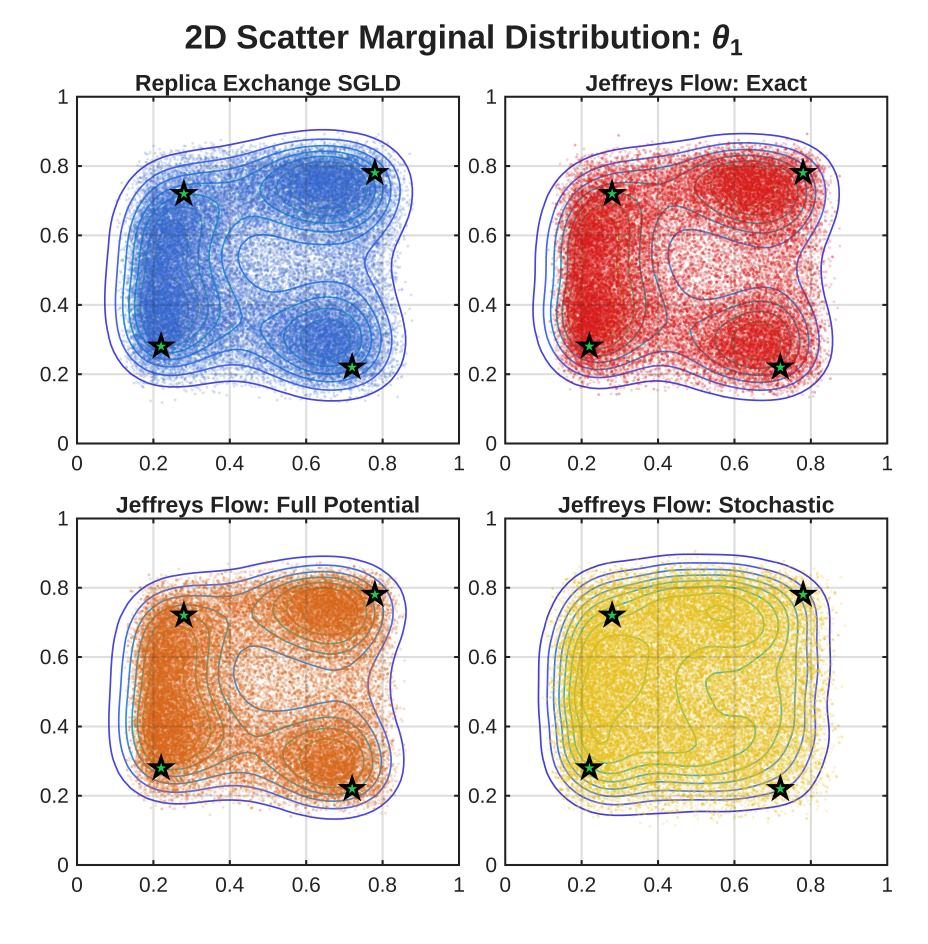}
	\caption{Marginal Posterior for Source 1 ($\theta_1$) in the Poisson Inverse Problem. The Exact configuration produces tightly bounded, high-fidelity posterior samples. The Neural Network-approximated Stochastic result exhibits structural degradation, validating the necessity of using exact PDE evaluations for calculating the importance weights.}
	\label{fig: poisson_marginals}
\end{figure*}

\begin{table*}[htbp]
	\centering
	\setlength{\tabcolsep}{6pt}
	\setstretch{1.2}
	\begin{tabular}{c |*{7}{c}}
		\hline
		Discretization $N$ & 8 & 12 & 16 & 20 & 24 & 28 & 32 \\
		\hline\hline
		$L^2$ Bias & 1.60e-02 & 7.45e-03 & 5.93e-03 & 3.56e-03 & 4.02e-03 & 2.00e-03 & 3.84e-04 \\
		\hline
	\end{tabular}
	\caption{Sampling Bias Evaluation in High Dimensions. Using solely a single transport map trained exclusively at the truncated dimension $N_0 = 8$, the reweighted Jeffreys Flow generates high-fidelity sampling accuracy uniformly up to $N = 32$, with the bias adhering to the rapid theoretical $\mathcal{O}(1/N^2)$ discretization error decay.}
	\label{tab: pimc_bias}
\end{table*}

\begin{figure*}
	\centering
	\includegraphics[width=0.6\textwidth]{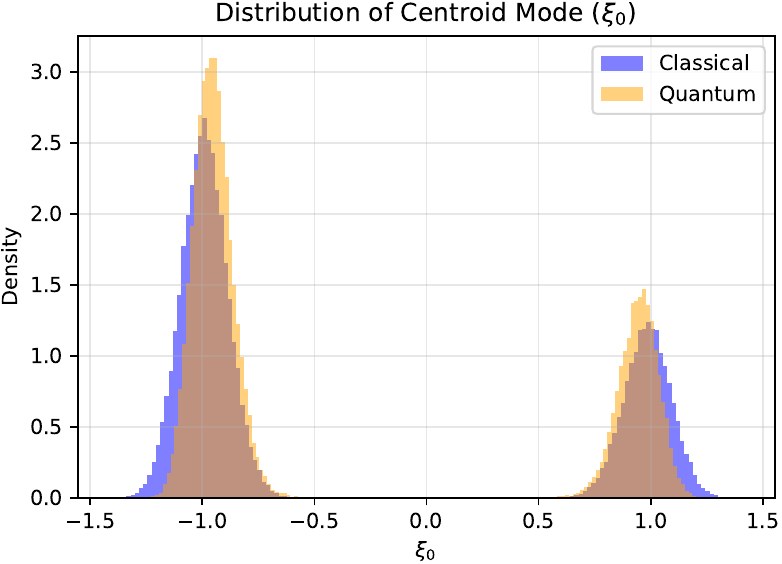}
	\caption{Quantum Path Integral Distillation. Comparison of the marginal spatial density between the classical distribution constructed from the empirical training samples and the distilled quantum distribution extracted using Jeffreys Flow.}
	\label{fig: pimc_density}
\end{figure*}

\begin{figure*}
	\centering
	\includegraphics[width=0.6\textwidth]{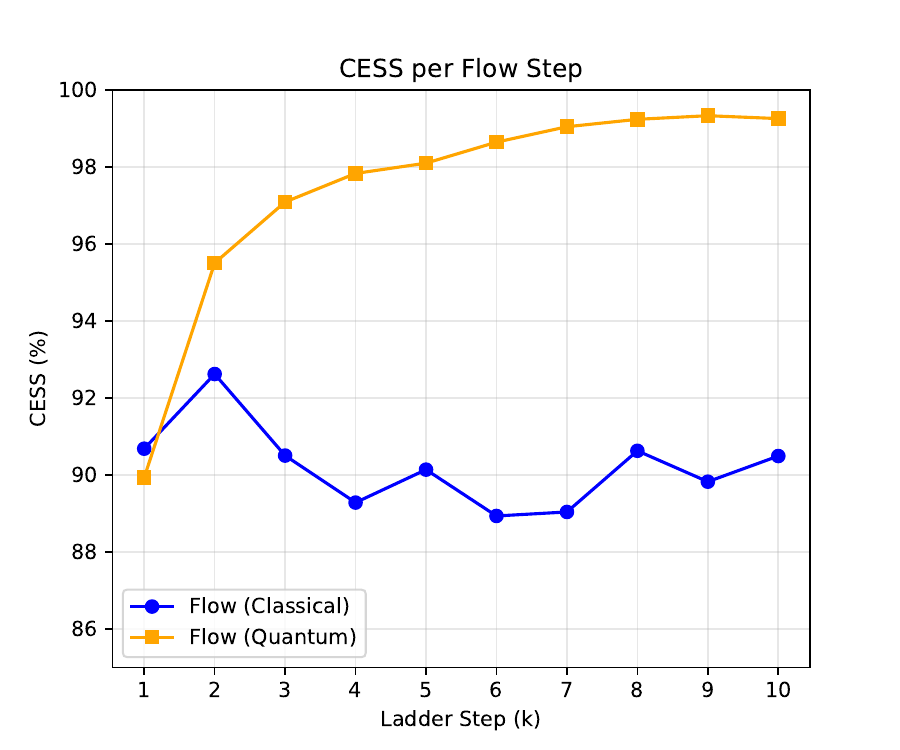}
	\caption{Progression of the CESS over the intermediate bridging steps. Sustained high CESS scores confirm the robustness and stability of the transport map during the classical-to-quantum distillation sequence.}
	\label{fig: pimc_cess}
\end{figure*}

\end{document}